\documentclass[11pt]{article}
\usepackage[margin=1in]{geometry}
\usepackage{natbib}
\setcitestyle{open={(},close={)}}
\usepackage[colorlinks=true,citecolor=blue]{hyperref}

\usepackage[american]{babel}
\usepackage{natbib} % has a nice set of citation styles and commands
\usepackage{mathtools} % amsmath with fixes and additions
\usepackage{booktabs} % commands to create good-looking tables
\usepackage{tikz} % nice language for creating drawings and diagrams

\DeclareMathAlphabet{\mathcal}{OMS}{cmsy}{m}{n}
\usepackage{url}            % simple URL typesetting
\usepackage{booktabs}       % professional-quality tables
\usepackage{amsfonts}       % blackboard math symbols
\usepackage{nicefrac}       % compact symbols for 1/2, etc.
\usepackage{microtype}      % microtypography

\usepackage{graphicx}
\usepackage{amssymb}
\usepackage{amsmath} \allowdisplaybreaks[4]
\usepackage{algorithm}
\usepackage{algorithmic}

\usepackage{amssymb}
\usepackage{amsthm}
\usepackage{multirow}
\usepackage{caption,subcaption}
\usepackage{mathrsfs}
\usepackage{bbm}
\usepackage{bm}
\usepackage{eufrak}
\usepackage{bookmark}
\usepackage{mathtools}
\usepackage{enumitem}
\newcommand{\circled}[1]{\small{\raisebox{.6pt}{\textcircled{\raisebox{-.8pt}{#1}}}}}
\graphicspath{{illustrations/}}

\def\aff#1{\textup {aff}\left(#1\right)}

\newcounter{optproblem}

\newtheoremstyle{mytheoremstyle} % name
    {\topsep}                    % Space above
    {\topsep}                    % Space below
    {\normalfont}                % Body font
    {}                           % Indent amount
    {\bfseries}                   % Theorem head font
    {.}                          % Punctuation after theorem head
    {.5em}                       % Space after theorem head
    {}  % Theorem head spec (can be left empty, meaning ânormalâ?)

\theoremstyle{mytheoremstyle}
\newtheorem{theorem}{Theorem}[section]
\newtheorem{remark}[theorem]{Remark}

\newtheorem*{theorem*}{Theorem}
\newtheorem*{lemma*}{Lemma}
\newtheorem*{remark*}{Remark}
\newtheorem{lemma}[theorem]{Lemma}%[section]

\theoremstyle{mytheoremstyle}
\newtheorem{definition}{Definition}[section]

\theoremstyle{remark}
\DeclareMathAlphabet{\pazocal}{OMS}{zplm}{m}{n}
\DeclareMathAlphabet{\mathpzc}{OMS}{pzc}{m}{it}

\setlist[itemize]{leftmargin=*}

\renewcommand{\hat}{\widehat}

\newcommand{\bfm}[1]{\ensuremath{\mathbf{#1}}}
\newcommand{\bfsym}[1]{\ensuremath{\boldsymbol{#1}}}

   \def\bA{\bfm A}  
   \def\bB{\bfm B}  
   \def\bC{\bfm C}  
     
   \def\bE{\bfm E}  
  \def\bF{\bfm F}  
     
   \def\bH{\bfm H}  
   \def\bI{\bfm I}

\def\bn{\bfm n}   \def\bN{\bfm N}  
     
   \def\bP{\bfm P}  
   \def\bQ{\bfm Q}  
   \def\bR{\bfm R}  \def\RR{\mathbb{R}}
   \def\bS{\bfm S}  
   \def\bT{\bfm T}  
\def\bu{\bfm u}   \def\bU{\bfm U}  
\def\bv{\bfm v}   \def\bV{\bfm V}  
   \def\bW{\bfm W}  
\def\bx{\bfm x}   \def\bX{\bfm X}  
\def\by{\bfm y}   \def\bY{\bfm Y}  
   \def\bZ{\bfm Z}  
\def\bzero{\bfm 0}

 \def\cH{{\cal  H}}

 \def\cN{{\cal  N}}
 \def\cO{{\cal  O}}

 \def\cS{{\cal  S}}

\def\balpha{\bfsym \alpha}
\def\bbeta{\bfsym \beta}

\def\bSigma{\bfsym \Sigma}

\def\+#1{\mathcal{#1}}
\def\-#1{\textup{#1}}

\def\set#1{\left\{ #1 \right\}}
\def\pth#1{\left( #1 \right)}
\def\bth#1{\left[ #1 \right]}
\def\abth#1{\left | #1 \right |}

\def\defeq {\coloneqq}

\newcommand{\overbar}[1]{\mkern 1.5mu\overline{\mkern-1.5mu#1\mkern-1.5mu}\mkern 1.5mu}

\newcommand{\La}{\left\langle\kern-0.64ex\left\langle}
\newcommand{\Ra}{\right\rangle\kern-0.64ex\right\rangle}

\def\Norm#1#2{{\left\vert\kern-0.4ex\left\vert\kern-0.4ex\left\vert #1
    \right\vert\kern-0.4ex\right\vert\kern-0.4ex\right\vert}_{#2}}
\def\norm#1#2{{\left\|#1\right\|}_{#2}}

\def\lonenorm#1{\norm{#1}{1}}
\def\ltwonorm#1{\norm{#1}{2}}

\def \supp#1{\textup{supp}\left(#1\right)}

\newcommand{\1}{{\rm 1}\kern-0.25em{\rm I}}
\def\indict#1{{\rm 1}\kern-0.25em{\rm I}_{\{#1\}}}

\DeclarePairedDelimiter\floor{\lfloor}{\rfloor}

\def \eps  {\epsilon}
\def \eps {\varepsilon}

\def\set#1{\left\{#1\right\}}

\def\unitsphere#1{\mathbb{S}^{#1}}
\def \E {\mathbb{E}}
\def\Expect#1#2{\E_{#1}\left[#2\right]}

\def \Pr {\textup{Pr}}
\newcommand{\Prob}[1]{\Pr\left[#1\right]}

\newcommand{\beq}{\begin{equation}}
\newcommand{\eeq}{\end{equation}}
\newcommand{\beqa}{\begin{eqnarray}}
\newcommand{\eeqa}{\end{eqnarray}}
\newcommand{\beqas}{\begin{eqnarray*}}
\newcommand{\eeqas}{\end{eqnarray*}}
\def\bal#1\eal{\begin{align}#1\end{align}}
\def\bals#1\eals{\begin{align*}#1\end{align*}}
\def\bsal#1\esal{\begin{small}\begin{align}#1\end{align}\end{small}}
\def\bsals#1\esals{\begin{small}\begin{align*}#1\end{align*}\end{small}}
\def\bsfal#1\esfal{\begin{small}\begin{flalign}#1\end{flalign}\end{small}}

\begin{document}

\title{\bf Noisy $\ell^{0}$-Sparse Subspace Clustering on Dimensionality Reduced Data}

\author{\vspace{0.5in}\\
 {\bf Yingzhen Yang} \\
School of Computing and Augmented Intelligence\\
Arizona State University\\
699 S Mill Ave. Tempe, AZ 85281, USA\\
\texttt{yingzhen.yang@asu.edu}\\\\
\and
\textbf{Ping Li} \\
Cognitive Computing Lab\\
Baidu Research\\
10900 NE 8th St. Bellevue, WA 98004, USA\\
  \texttt{pingli98@gmail.com}
}
\date{\vspace{0.5in}}

\maketitle

\begin{abstract}\vspace{0.3in}
\noindent\footnote{Yingzhen Yang's work was conducted as a consulting researcher at Baidu Research - Bellevue,~WA,~USA.\vspace{-0.1in}}Sparse subspace clustering methods with sparsity induced by $\ell^{0}$-norm, such as $\ell^{0}$-Sparse Subspace Clustering ($\ell^{0}$-SSC)~\citep{YangFJYH16-L0SSC-ijcv}, are demonstrated to be more effective than its $\ell^{1}$ counterpart such as Sparse Subspace Clustering (SSC)~\citep{ElhamifarV13}. However, the theoretical analysis of $\ell^{0}$-SSC is restricted to clean data that lie exactly in subspaces. Real data often suffer from noise and they may lie close to subspaces. In this paper, we show that an optimal solution to the optimization problem of noisy $\ell^{0}$-SSC achieves subspace detection property (SDP), a key element with which data from different subspaces are separated, under deterministic and semi-random model. Our results provide theoretical guarantee on the correctness of noisy $\ell^{0}$-SSC in terms of SDP on noisy data for the first time, which reveals the advantage of noisy $\ell^{0}$-SSC in terms of much less restrictive condition on subspace affinity. In order to improve the efficiency of noisy $\ell^{0}$-SSC, we propose Noisy-DR-$\ell^{0}$-SSC which provably recovers the subspaces on dimensionality reduced data. Noisy-DR-$\ell^{0}$-SSC first projects the data onto a lower dimensional space by random projection, then performs noisy $\ell^{0}$-SSC on the projected data for improved efficiency. Experimental results demonstrate the effectiveness of Noisy-DR-$\ell^{0}$-SSC.
\end{abstract}

\newpage

\section{Introduction}

Clustering is an important unsupervised learning procedure for analyzing a broad class of scientific data. High-dimensional data, such as facial images and gene expression data, often lie in low-dimensional subspaces in many cases, and clustering in accordance to the underlying subspace structure is particularly important. Among various subspace clustering algorithms, the ones that employ sparsity prior, such as Sparse Subspace Clustering (SSC)~\citep{ElhamifarV13} and $\ell^{0}$-Sparse Subspace Clustering ($\ell^{0}$-SSC)~\citep{YangFJYH16-L0SSC-ijcv}, have been proven to be effective in separating the data in accordance with the subspaces that the data lie in under certain assumptions. Furthermore, Sparse Additive Subspace Clustering~\citep{yuan2014sparse} considers a nonlinear transformation of each data point such that the transformed point can be linearly represented by data in the same subspace as that point, extending the usual linear representation by SSC.

Sparse subspace clustering methods construct the sparse similarity matrix by sparse representation of the data. Subspace detection property (SDP) defined in Section~\ref{sec::setup-method} ensures that the similarity between data from different subspaces vanishes in the sparse similarity matrix, and applying spectral clustering~\citep{Ng01} on such sparse similarity matrix leads to compelling clustering performance. Elhamifar and Vidal~\citep{ElhamifarV13} prove that when the subspaces are independent or disjoint, SDP can be satisfied by solving the canonical sparse linear representation problem using data as the dictionary, under certain conditions on the rank, or singular value of the data matrix and the principle angle between the subspaces. Under the independence assumption on the subspaces, low rank representation~\citep{LiuLY13,liu2014recovery,liu2016low} is also proposed to recover the subspace structures. Relaxing the assumptions on the subspaces to allowing overlapping subspaces, the Greedy Subspace Clustering~\citep{ParkCS14} and the Low-Rank Sparse Subspace Clustering~\citep{Wang13-lrr-ssc} achieve subspace detection property with high probability. The geometric analysis in~\citet{Soltanolkotabi2012} shows the theoretical results on subspace recovery by SSC. In the following, we use the term SSC or $\ell^{1}$-SSC exchangeably to indicate the Sparse Subspace Clustering method in~\citet{ElhamifarV13}.

Real data often suffer from noise. The correctness of noisy SSC is analyzed in~\citet{WangX13} which handles noisy data that lie close to disjoint or overlapping subspaces, and the original optimization problem of noisy SSC is proposed in~\citet{ElhamifarV13}. While~\citet{YangFJYH16-L0SSC,Yang18-DRL0SSC} prove the correctness of $\ell^{0}$-SSC or its dimensionality reduced variant on clean data based on a constrained $\ell^0$-minimization problem, it empirically solves an unconstrained $\ell^{0}$-regularized problem to handle noise in data, and they lack theoretical analysis on the correctness of $\ell^{0}$-SSC on noisy data. While $\ell^{0}$-SSC~\citep{YangFJYH16-L0SSC-ijcv} has guaranteed clustering correctness via subspace detection property under much milder assumptions than previous subspace clustering methods including SSC, it assumes that the observed data lie in exactly in the subspaces and does not handle noisy data.

In this paper, we present noisy $\ell^{0}$-SSC, which enhances $\ell^{0}$-SSC by theoretical guarantee on the correctness of clustering on noisy data. It should be emphasized that while $\ell^{0}$-SSC on clean data~\citep{YangFJYH16-L0SSC-ijcv} empirically adopts a form of optimization problem robust to noise, it lacks theoretical analysis on the correctness of $\ell^{0}$-SSC on noisy data.
In this paper, the correctness of noisy $\ell^{0}$-SSC on noisy data in terms of the subspace detection property is established. Our analysis is under both deterministic model and semi-random model, which are the models employed in the geometric analysis of SSC~\citep{Soltanolkotabi2012}. Our randomized analysis demonstrates the advantage of noisy $\ell^{0}$-SSC over its $\ell^{1}$ counterpart as more general assumption on data distribution can be adopted. Moreover, we present Noisy Dimensionality Reduced $\ell^{0}$-Sparse Subspace Clustering (Noisy-DR-$\ell^{0}$-SSC), an efficient version of noisy $\ell^{0}$-SSC which also enjoys robustness to noise. Noisy-DR-$\ell^{0}$-SSC first projects the data onto a lower dimensional space by random projection, then performs noisy $\ell^{0}$-SSC on the projected data. Noisy-DR-$\ell^{0}$-SSC provably recovers the underlying subspace structure in the original data from the projected data under deterministic model. Experiments show the effectiveness of  noisy $\ell^{0}$-SSC and Noisy-DR-$\ell^{0}$-SSC.

\vspace{-0.03in}
\subsection{Notations}
\vspace{-0.03in}

We use bold letters for matrices and vectors, and regular lower letter for scalars throughout this paper. The bold letter with superscript indicates the corresponding column of a matrix, e.g. $\bA^i$ is the $i$-th column of matrix $\bA$, and the bold letter with subscript indicates the corresponding element of a matrix or vector. $\|\cdot\|_F$ and $\|\cdot\|_p$ denote the Frobenius norm and the vector $\ell^{p}$-norm or the matrix $p$-norm, and $\norm{\cdot}{0}$ is the $\ell^{0}$-norm, that is, the number of nonzero elements of a vector. ${\rm diag}(\cdot)$ indicates the diagonal elements of a matrix. $\bH_{\bT} \subseteq \RR^d$ indicates the subspace spanned by the columns of $\bT$, and $\bA_{\bI}$ denotes a submatrix of $\bA$ whose columns correspond to the nonzero elements of $\bI$ (or with indices in $\bI$ without confusion). $\sigma_{t}(\cdot)$ denotes the $t$-th largest singular value of a matrix, and $\sigma_{\min}(\cdot)$ indicates the smallest singular value of a matrix. ${\rm supp}(\cdot)$ is the support of a vector, $\mathbb P_{\cS'}$ is the operator of orthogonal projection onto the subspace $\cS'$. $[n]$ represents all the natural numbers between $1$ and $n$ inclusively. $\unitsphere{d-1}$ denotes the unit sphere in $\RR^{d}$. $\Theta(a)$ denotes a number such that there exists two constants $c_1$ and $c_2$ such that $\Theta(a) \in [c_1a,c_2a]$.

\subsection{Contributions}\label{sec::contributions}

\begin{table*}[b!]

\vspace{-0.1in}

\centering
\scriptsize
\caption{\small Comparison between Different Subspace Clustering Methods in terms of Conditions Required under the Semi-Random Model: Greedy Subspace Clustering (GSC)~\citep{ParkCS14}; $\ell^1$-SSC~\citep{ElhamifarV11}, Noisy SSC~\citep{WangX13}; Affine Sparse Subspace Clustering (ASSC)~\citep{Li2018-ASSC,YouLRV19-ASSC}; Noisy $\ell^{0}$-SSC~\citep{YangFJYH16-L0SSC-ijcv}. Please refer to Section~\ref{sec::contributions} for the definition of notations.}
\resizebox{1.0\linewidth}{!}{
\begin{tabular}{|c|c|c|}
  \hline
   Methods or Assumptions &Allowing Overlapping Subspaces &Subspace Affinity \\ \hline
GSC &Yes &$\max_{k,l \in [K]}\text{aff}(\mathcal S_k,S_l) < \frac{C_2\log {n/K}}{\log {(d_0L \delta^{-1})} \cdot \log {(nd_0\delta^{-1}})} \overset{d_0 \to \infty}{\longrightarrow} 0$ \\ \hline
$\ell^1$-SSC, Noisy SSC &Yes & $\max_{k,l \in [K]}\text{aff}(\mathcal S_k,S_l) < \sqrt{d_0} \cdot \frac{\bar c \sqrt{\log{\rho}}}{8{\sqrt 2} \log{n}}   \overset{d_0 \to \infty}{\longrightarrow} 0$ \\ \hline
ASSC &No & NA \\ \hline
Noisy $\ell^{0}$-SSC &Yes &$\max_{k,l \in [K]}\text{aff}(\mathcal S_k,S_l) <   \frac{\sigma'^2_{\min}}{r_0-1} > 0$ for sufficiently large $d_0$ \\ \hline
\end{tabular}
}
\label{table:theoretical-results}\vspace{-0.3in}
\end{table*}

First, the correctness of noisy $\ell^{0}$-SSC on noisy data in terms of the subspace detection property is established for the first time, which is presented in Section 3 of this paper. Our analysis is under both deterministic model and semi-random model, which are also the models employed by the geometric analysis of SSC [Soltanolkotabi2012]. Our randomized analysis demonstrates the significant advantage of noisy $\ell^{0}$-SSC over its $\ell^{1}$ counterpart and other competing subspace clustering methods in terms of much less restrictive condition on the subspace affinity. Table~\ref{table:theoretical-results} below demonstrates the conditions under which SDP holds for representative subspace clustering methods under the semi-random model, including Greedy Subspace Clustering (GSC)~\citep{ParkCS14},  $\ell^1$-SSC~\citep{ElhamifarV11}, Noisy SSC~\citep{WangX13}, Affine Sparse Subspace Clustering (ASSC)~\citep{Li2018-ASSC,YouLRV19-ASSC}, Noisy $\ell^{0}$-SSC~\citep{YangFJYH16-L0SSC-ijcv}. When the size of data $n$ grows exponentially in terms of the common subspace dimension $d_0$ (the dimension of every subspace is $d_0$), in particular, $n = \Theta(e^{d_0^{\tau}})$  for $\tau \in (0.5,0.9)$, then all the competing subspace clustering methods other than Noisy $\ell^{0}$-SSC either do not allow overlapping subspaces, or require the maximum pairwise subspace affinity goes to $0$ when $d_0 \to \infty$, which means that these methods require all the subspaces to be almost pairwise orthogonal when the common subspace dimension $d_0$ is very large. Instead, Noisy $\ell^{0}$-SSC allows  subspace affinity to be lower bounded from $0$, suggesting that Noisy $\ell^{0}$-SSC is still able to recover subspaces which are not orthogonal when $d_0$ is very large. $\sigma'_{\min}$ in Table~\ref{table:theoretical-results} is defined in Theorem~\ref{theorem::noisy-l0ssc-subspace-detection-lambda-random}.

In Table~\ref{table:theoretical-results}, it is preferred that a subspace clustering method requires milder conditions, which are allowing overlapping subspaces and larger upper bound for the maximum subspace affinity, denoted by $\max_{k,l \in [K]}\text{aff}(\mathcal S_k,S_l)$ where $\set{\cS_k}_{k=1}^K$ are $K$ subspaces, so that the underlying subspaces can be recovered for overlapping subspaces and for subspaces which are closer to each other (larger subspace affinity). Here $\text{aff}$ denotes subspace affinity, $d_0$ is the common subspace dimension, $\delta$ is a small positive constant (see~\citep{ParkCS14}). $r_0 > 1$ is an upper bound for the support of an optimal solution to the noisy $\ell^0$-SSC problem for all data points.  In this table, $n = \Theta(e^{d_0^{\tau}})$ for $\tau \in (0.5,0.9)$ when $d_0 \to \infty$. Note that two subspaces are overlapping subspaces if the dimension of their intersection is larger than $1$. When a subspace clustering method does not allow overlapping subspaces, then no condition on subspace affinity is presented in the subspace clustering literature.

Second, we propose Noisy Dimensionality Reduced $\ell^{0}$-Sparse Subspace Clustering (Noisy-DR-$\ell^{0}$-SSC) to accelerate noisy $\ell^{0}$-SSC with provable robustness to noise. Noisy-DR-$\ell^{0}$-SSC first projects the data onto a lower dimensional space by random projection, then performs noisy $\ell^{0}$-SSC on the dimensionality reduced data. Two types of random projections are used in Noisy-DR-$\ell^{0}$-SSC, which are the random projection induced by randomized low-rank approximation and the sparse random projection, particular "Count-Sketch (SC) Projections''.

It should be emphasized that~\citet{Yang18-DRL0SSC} also studies dimensionality reduced $\ell^{0}$-SSC. However, the analysis of that work is performed on the following constrained $\ell^0$-minimization problem only for clean data without noise,
\bals
\mathop {\min }\limits_{{\bZ}}  \norm{\bZ}{0} \quad s.t.\;{\tilde \bX} = {{\tilde \bX}}{\bZ},\,\, {\rm diag}(\bZ) = \bzero,
\eals%
where $\tilde \bX$ is the dimensionality reduced version of clean data by random projection. On the other hand, the actual optimization problem of ~\citep{Yang18-DRL0SSC} solves the following unconstrained $\ell^{0}$-regularized problem,
\bal\label{eq:noisy-dr-l0ssc-intro}%
\mathop {\min }\limits_{{\bZ}} \ltwonorm{\tilde \bX - \tilde \bX \bZ}^2 + \lambda \norm{\bZ}{0}\quad s.t.\;,\,\, {\rm diag}(\bZ) = \bzero,
\eal%
without guarantee of any solution to (\ref{eq:noisy-dr-l0ssc-intro}). In contrast, we analyze noisy $\ell^{0}$-SSC on the unconstrained $\ell^{0}$-regularized problem (\ref{eq:noisy-dr-l0ssc-intro}) with noisy data which reveals the advantage of noisy $\ell^0$ SSC over $\ell^1$-SSC. Our analysis also suggests that a larger $\lambda$ tends to guarantee the subspace detection property (Remark~\ref{remark::lambda}), verified by experiments. Throughout the paper, we refer to $\ell^{0}$-SSC for noisy data with the unconstrained $\ell^{0}$-regularized problem as noisy $\ell^{0}$-SSC.

\section{Problem Setup}\label{sec::setup-method}

Sparse Subspace Clustering (SSC) methods, such as~\citet{ElhamifarV11,Soltanolkotabi2012,WangX13,yuan2014sparse}, construct a sparse similarity matrix by sparse representation of the data, and then perform clustering on the sparse similarity matrix.

We hereby introduce the notations for subspace clustering on noisy data considered in this paper. The uncorrupted data matrix is denoted by ${\bY}=[ {{\by_1},\ldots ,{\by_n}} ] \in {\RR^{d \times n}}$, where $d$ is the dimensionality and $n$ is the size of the data. The uncorrupted data $\bY$ lie in a union of $K$ distinct subspaces $\{\cS_k\}_{k=1}^K$ of dimensions $\{d_k\}_{k=1}^K$ with $d_{\max} \defeq \max_{k \in [K]} d_k$ and $d_{\min} \defeq \min_{k \in [K]} d_k$. The observed noisy data is $\bX = \bY + \bN$, where ${\bN}=[ {{\bn_1},\ldots ,{\bn_n}} ] \in {\RR^{d \times n}}$ is the additive noise. $\bx_i = \by_i + \bn_i$ is the noisy data point that is corrupted by the noise $\bn_i$. We let $\bY^{(k)} \in \RR^{d \times n_k}$ denote the data belonging to subspace $\cS_k$ with $\sum\limits_{k=1}^K n_k = n$, and denote the corresponding columns in $\bX$ by $\bX^{(k)}$. Let $\bU^{(k)} \in \RR^{d \times d_k}$ be the orthogonal basis of $\cS_k$ for all $k \in [K]$. The data $\bX$ are normalized such that each column has unit $\ell^{2}$-norm in our deterministic analysis. We consider deterministic noise model where the noise $\bN$ is fixed and $\max_{i \in [n]}\ltwonorm{\bn_i} \le \delta$.

\vspace{0.1in}

Formally, given observed data $\bX \in \RR^{d \times n}, \bX = \bth{\bx_1, \ldots, \bx_n}$, where $\bx_i \in \RR^d$, SSC solves the following optimization problem for each $i \in [n]$:
\bsal\label{eq:ssc-opt}
\min_{\bbeta \in \RR^n} {\lonenorm{\bbeta}} \quad \text{s.t. } \bx_i =  \bX \bbeta, \bbeta_i = 0.
\esal

In order to handle noisy data, noisy SSC~\citep{WangX13} was proposed to solve the $\ell^{1}$ regularized problem:
\bsal\label{eq:noisy-ssc-opt}
\min_{\bbeta \in \RR^n} ||\bx_i -  \bX \bbeta||^2 + \lambda \lonenorm{\bbeta}, \quad \text{s.t. }  \bbeta_i = 0.
\esal%
The sparse code $\bbeta$ of the data point $\bx_i$ is obtained by solving (\ref{eq:ssc-opt}) or (\ref{eq:noisy-ssc-opt}) for SSC or noisy SSC.  A coefficient matrix $\bZ \in \RR^{n \times n}$ is then formed by concatenating the sparse codes of all the data points, and the $i$-th column of $\bZ$ is the sparse code of $\bx_i$. The sparse similarity matrix is then computed~by $\bW = \frac{|\bZ | + |\bZ ^{\top}|}{2}$. Subspace detection property (SDP, formally defined later) ensures that the similarity between data from different subspaces vanishes in the sparse similarity matrix. If SDP holds, then similarity between~data points from different clusters vanish in $\bW$. As a result, performing spectral clustering on $\bW$ leads to compelling clustering~results.

Under the independence assumption on the subspaces, low rank representation~\citep{LiuLY13} is proposed to recover the subspace structures. Relaxing the assumptions on the subspaces to allowing overlapping subspaces, the Greedy Subspace Clustering~\citep{ParkCS14} and the Low-Rank Sparse Subspace Clustering~\citep{LiuLY13} achieve subspace detection property with high probability. The geometric analysis in~\citet{Soltanolkotabi2012} shows the theoretical results on subspace recovery by SSC. In the following text, we use  SSC or $\ell^{1}$-SSC exchangeably to indicate the Sparse Subspace Clustering method in~\citet{Soltanolkotabi2012,ElhamifarV13}.

$\ell^{0}$-SSC~\citep{YangFJYH16-L0SSC-ijcv} proposes to solve the following $\ell^{0}$ sparse representation problem
\bsal\label{eq:l0ssc}
\min_{\bZ \in \RR^{n \times n}} \norm{\bZ}{0} \quad s.t. \,\, \bX = \bX \bZ, \,\, {\textup {diag}}(\bZ)= \bzero,
\esal%
and it proves that SDP is satisfied with an globally optimal solution to problem (\ref{eq:l0ssc}). In~\citet{YangFJYH16-L0SSC-ijcv}, the $\ell^{0}$ regularized sparse approximation problem below is solved so as to handle noisy data for $\ell^{0}$-SSC, which is the optimization problem of noisy $\ell^{0}$-SSC:
\bsal\label{eq:l0ssc-lasso}
\min_{\bZ \in \RR^{n \times n}, {\textup {diag}}(\bZ) = 0}  L(\bZ) = \norm{\bX - \bX \bZ}{F}^2 + \lambda\norm{\bZ}{0}, .
\esal%
The optimization problem of noisy $\ell^{0}$-SSC (\ref{eq:l0ssc-lasso}) is separable. For each $i \in [n]$, the optimization problem with respect to the sparse code $\bbeta$ of $i$-th data point is
\bal\label{eq:noisy-l0ssc-i}
&\mathop {\min }\limits_{{\bbeta \in \RR^n,\bbeta_i = 0}} L(\bbeta) = {\|\bx_i - \bX \bbeta\|_2^2 + {\lambda}\|{\bbeta}\|_0}.
\eal%%
The sparse similarity matrix $\bW$ is then computed in the same way as $\ell^{1}$-SSC by $\bW = \frac{|\bZ| + |\bZ ^{\top}|}{2}$, and the subspace clustering result of noisy $\ell^{0}$-SSC is achieved by performing spectral clustering on $\bW$.

In the following text, we always use $\bbeta^*$ to denote an optimal solution to (\ref{eq:noisy-l0ssc-i}), and define $r^* \defeq \norm{\bbeta^*}{0}$.

\newpage

The definition of subspace detection property for noisy $\ell^{0}$-SSC and noiseless $\ell^{0}$-SSC, i.e. $\ell^{0}$-SSC on noiseless data, is defined in Definition~\ref{def::subspace-detection} below.

\begin{definition}\label{def::subspace-detection}
({Subspace detection property for noisy and noiseless $\ell^{0}$-SSC})
Let $\bZ^*$ be an optimal solution to (\ref{eq:l0ssc-lasso}). The subspaces $\{\cS_k\}_{k=1}^K$ and the data $\bX$ satisfy the Subspace Detection Property (SDP) for noisy $\ell^{0}$-SSC if $\bZ^i$ is a nonzero vector, and nonzero elements of $\bZ^i$ correspond to the columns of $\bX$ from the same subspace as $\by_i$ for all $1 \le i \le n$.
We say that SDP holds for $\bx_i$ if nonzero elements of ${\bZ^*}^i$, which is $\bbeta^*$ for problem (\ref{eq:noisy-l0ssc-i}), correspond to the data that lie in the same subspace as $\by_i$, for either noisy $\ell^{0}$-SSC or noiseless $\ell^{0}$-SSC.
\end{definition}

\section{Analysis for Noisy $\ell^{0}$-SSC}
\label{sec::noisy-l0ssc-correctness}

Similar to~\citet{Soltanolkotabi2012}, we introduce the deterministic and the semi-random model for the analysis of noisy $\ell^{0}$-SSC.
\begin{itemize}%[leftmargin=*]
\item {\textbf{Deterministic Model:}} the subspaces and the data in each subspace are fixed.
\item {\textbf{Semi-Random Model:}} the subspaces are fixed but the data are independent and identically distributed in each of the subspaces.
\end{itemize}
The data in the above definitions refer to clean data without noise. Both the deterministic model and the semi-random model are extensively employed to analyze the subspace detection property in the subspace learning literature~\citep{Soltanolkotabi2012,Wang13-lrr-ssc,WangX13,Wang16-graphconnectivity}.

\subsection{Noisy $\ell^{0}$-SSC: Deterministic Analysis}

We first introduce the definition of general position and external subspace before our analysis on noisy $\ell^{0}$-SSC.
\begin{definition}\label{def::general-position}
{\rm (General position)}
For any $1 \le k \le K$, the data $\bY^{(k)}$ are in general position if any subset of $L \le d_k$ data points (columns) of $\bY^{(k)}$ are linearly independent. $\bY$ are in general position if $\bY^{(k)}$ are in general position for $1 \le k \le K$.
\end{definition}
The assumption of general condition is rather mild. In fact, if the data points in $\bX^{(k)}$ are independently distributed according to any continuous distribution, then they almost surely in general position.

Let the distance between a point $\bx \in \RR^d$ and a subspace $\cS \subseteq \RR^d$ be defined as $d(\bx,\cS) = \inf_{\by \in \cS} \|\bx-\by\|_2$, the definition of external subspaces is presented as follows.
\begin{definition}\label{def::external-subspace}
{\rm (External subspace of limited dimension)}
For a point $\by \in \bY^{(k)}$, a subspace $\bH_{\{\by_{i_j}\}_{j=1}^L}$ spanned by a set of linear independent points $\{\by_{i_j}\}_{j=1}^L \subseteq \bY$ is defined to be an external subspace of $\by$ if $\{\by_{i_j}\}_{j=1}^L \not \subseteq \bY^{(k)}$ and $\by \notin \{\by_{i_j}\}_{j=1}^L$. The set of all external subspaces of $\by$ of dimension no greater than $r$ with $r \ge 1$ for $\by$ is denoted by $\cH_{\by,r}$, that is, $\cH_{\by,r} = \{\bH \colon \bH = \bH_{\{\by_{i_j}\}_{j=1}^L}, {\rm dim}[\bH] = L, L \le r, \{\by_{i_j}\}_{j=1}^L \not \subseteq \bY^{(k)}, \by \notin \{\by_{i_j}\}_{j=1}^L \}$. The point $\by$ is said to be away from its external subspaces of dimension $r$ if $\min_{\bH \in \cH_{\by, r}} d(\by, \bH) > 0$. All the data points in $\bY^{(k)}$ are said to be away from the external subspaces if each of them is away from the its associated external spaces.
\end{definition}

We also need the definitions related to the spectrum of $\bX$ and $\bY$, which are defined as follows. In the following analysis, we employ $\bbeta$ to denote the sparse code of datum $\bx_i$ so that a simpler notation other than $\bZ^i$ is dedicated to our analysis.

\begin{definition}\label{def::minimum-eigenvalue}
The minimum restricted eigenvalue of the uncorrupted data is defined as
\bals
&\sigma_{\bY,r} \defeq \min_{\bbeta: \|\bbeta\|_0 = r, {\rm rank}(\bY_{\bbeta}) = \|\bbeta\|_0} \sigma_{\min}(\bY_{\bbeta})
\eals%%
for $r \ge 1$. In addition, the normalized minimum restricted eigenvalue of the uncorrupted data is defined by
\bals
&\bar \sigma_{\bY,r} \defeq \frac{\sigma_{\bY,r}}{\sqrt{r}}.
\eals%%
\end{definition}
Moreover, the following quantities are defined for our analysis. We define
\bal\label{eq:equivalence-noisy-l0ssc-tau0}
&\tau_0 \defeq \frac{2\delta\sqrt{r^*}}{\sigma_{\bX}^*} + \tau_1,
\eal%%
where
\bal\label{eq:equivalence-noisy-l0ssc-tau1}
&\tau_1 \defeq \frac{\delta} {\bar \sigma_{\bY}^* - \delta}, \quad \sigma_{\bX}^* \defeq \sigma_{\min}(\bX_{\bbeta^*}),
\eal%%
with $\delta < \bar \sigma_{\bY}^*$, and $\bar \sigma_{\bY}^*$ is defined as
\bal\label{eq:l0ssc-bar-sigma-star}
&\bar \sigma_{\bY}^* \defeq \min_{1 \le r < r^*} \bar \sigma_{\bY,r}.
\eal%%

Now we present our main result on  noisy $\ell^{0}$-SSC.
\begin{theorem}\label{theorem::noisy-l0ssc-subspace-detection}
{\rm (Subspace detection property holds for noisy $\ell^{0}$-SSC)}
Let nonzero vector $\bbeta^*$ be an optimal solution to the noisy $\ell^{0}$-SSC problem (\ref{eq:noisy-l0ssc-i}) for point $\bx_i$ with $\|\bbeta^*\|_0=r^* \ge 1$, and $c^* \defeq \|\bx_i - \bX \bbeta^*\|_2$. Suppose $\bY$ is in general position, $\by_i \in \cS_k$ for some $1 \le k \le K$, $\delta < \bar \sigma_{\bY}^*$, $\lambda > \tau_0$, $\bB(\by_i, \delta+c^*+\frac{2\delta\sqrt{r^*}}{\sigma_{\bX}^*}) \cap \bH = \emptyset$ for any $\bH \in \cH_{\by_i, r^*}$. Then the subspace detection property holds for $\bx_i$ with $\bbeta^*$. Here $\tau_0$, $\tau_1$, $\bar \sigma_{\bY}^*$ and $\sigma_{\bX}^*$ are defined in (\ref{eq:equivalence-noisy-l0ssc-tau0}), (\ref{eq:equivalence-noisy-l0ssc-tau1}) and (\ref{eq:l0ssc-bar-sigma-star}).
\end{theorem}
\begin{remark}
When $\delta = 0$ and there is no noise in the data $\bX$, the conditions for the correctness of noisy $\ell^{0}$-SSC in Theorem~\ref{theorem::noisy-l0ssc-subspace-detection} almost reduce to that for noiseless $\ell^{0}$-SSC. To see this, the conditions are reduced to $\bB(\by_i,c^*) \cap \bH = \emptyset$, which are exactly the conditions required by noiseless $\ell^{0}$-SSC in Lemma~\ref{lemma::l0ssc-deterministic}, namely data are away from the external subspaces by choosing $\lambda \to 0$ and it follows that $c^* = 0$.
\end{remark}

While Theorem~\ref{theorem::noisy-l0ssc-subspace-detection} establishes geometric conditions under which the subspace detection property holds for noisy $\ell^{0}$-SSC, it can be seen that these conditions are often coupled with an optimal solution $\bbeta^*$ to the noisy $\ell^{0}$-SSC problem (\ref{eq:noisy-l0ssc-i}). In the following theorem, the correctness of noisy $\ell^{0}$-SSC is guaranteed in terms of $\lambda$, the weight for the $\ell^{0}$ regularization term in (\ref{eq:noisy-l0ssc-i}), and the geometric conditions independent of an optimal solution to (\ref{eq:noisy-l0ssc-i}).

Let $M_i >0$ be the minimum distance between $\by_i \in \cS_k$ and its external subspaces when $\by_i$ is away from its external subspaces of dimension $r$, that is,
\bal\label{eq:Mi}
&M_{i} \defeq \min\{d(\by_i, \bH) \colon \bH \in \cH_{\by_i,d_k}\}.
\eal%%

The following two quantities related to the spectrum of clean and noisy data, $\mu_{r}$ and $\sigma_{\bX,r}$, are defined as follows with $r > 1$ for the analysis in Theorem~\ref{theorem::noisy-l0ssc-subspace-detection-lambda}.
\bal\label{eq:mu-r}
&\mu_{r} \defeq \frac{\delta} {\min_{1 \le r' < r} \bar \sigma_{\bY,r} - \delta},
\eal%%
\bal\label{eq:sigma-X-r}
&\sigma_{\bX,r} \defeq \min\{\sigma_{\min}(\bX_{\bbeta}) \colon 1 \le \|\bbeta\|_0 \le r\}
\eal%%

\begin{theorem}\label{theorem::noisy-l0ssc-subspace-detection-lambda}
{\rm (Subspace detection property holds for noisy $\ell^{0}$-SSC under deterministic model with conditions in terms of $\lambda$)}
Let nonzero vector $\bbeta^*$ be an optimal solution to the noisy $\ell^{0}$-SSC problem (\ref{eq:noisy-l0ssc-i}) for point $\bx_i$ with $\|\bbeta^*\|_0=r^* > 1$, $n_k \ge d_k+1$ for every $k \in [K]$, and there exists $1 < r_0 \le \floor{\frac{1}{\lambda}}$ such that $r^* \le r_0$. Suppose $\bY$ is in general position, $\by_i \in \cS_k$ for some $1 \le k \le K$, $\delta < \bar \sigma_{\bY}^*$, and $M_{i,\delta} \defeq M_i - \delta$. Suppose
\bal\label{eq:noisy-l0ssc-sdp-M}
&M_{i,\delta}  > \frac{2\delta}{\sigma_{\bX,r_0}},
\eal%%
and
\bal\label{eq:noisy-l0ssc-sdp-mu}
&\mu_{r_0} < 1-\frac{2\delta}{\sigma_{\bX,r_0}}.
\eal%%
Then if
\bal\label{eq:noisy-l0ssc-sdp-lambda}
& \lambda_{0} < \lambda < 1,
\eal%%
where $\lambda_0 \defeq \max\{\lambda_1,\lambda_2\}$ and
\bal
&\lambda_{1} \defeq \inf\{0 < \lambda < 1 \colon \sqrt{1-\lambda} + \frac{2\delta}{\sigma_{\bX,r_0} \sqrt{\lambda}}< M_{i,\delta}\}, \label{eq:noisy-l0ssc-sdp-min-lambda-1} \\
&\lambda_{2} \defeq \inf\{0 < \lambda < 1 \colon \lambda - \frac{2\delta}{\sigma_{\bX,r_0}} \frac{1}{\sqrt{\lambda}} > \mu_{r_0}\}, \label{eq:noisy-l0ssc-sdp-min-lambda-2}
\eal%%
the subspace detection property holds for $\bx_i$ with $\bbeta^*$. Here $M_i$, $\mu_{r_0}$, $\sigma_{\bX,r_0}$ are defined in (\ref{eq:Mi}), (\ref{eq:mu-r}), (\ref{eq:sigma-X-r}) respectively.
\end{theorem}
\begin{remark}
The two conditions (\ref{eq:noisy-l0ssc-sdp-M}) and (\ref{eq:noisy-l0ssc-sdp-mu}) are induced by two conditions, $\bB(\by_i, \delta+c^*+\frac{2\delta\sqrt{r^*}}{\sigma_{\bX}^*}) \cap \bH = \emptyset$ for any $\bH \in \cH_{\by_i, d_k}$ and $\lambda > \tau_0$ respectively, which are required by Theorem~\ref{theorem::noisy-l0ssc-subspace-detection}. Note that when (\ref{eq:noisy-l0ssc-sdp-M}) and (\ref{eq:noisy-l0ssc-sdp-mu}) hold, $\lambda_{1}$ and $\lambda_{2}$ can always be chosen in accordance with~(\ref{eq:noisy-l0ssc-sdp-min-lambda-1})~and~(\ref{eq:noisy-l0ssc-sdp-min-lambda-2}).
\end{remark}
\begin{remark}\label{remark::lambda}
It can be observed from condition (\ref{eq:noisy-l0ssc-sdp-lambda}) that noisy $\ell^{0}$-SSC encourages sparse solution by a relatively large $\lambda$ so as to guarantee the subspace detection property. This theoretical finding is consistent with the empirical study shown in the experimental results.
\end{remark}

\subsection{Noisy $\ell^{0}$-SSC: Randomized Analysis}
The correctness of noisy $\ell^{0}$-SSC is analyzed under the semi-random model that the clean data in subspace $\cS^{(k)}$ are i.i.d. according to the uniform distribution on the unit sphere, $\unitsphere{d_k-1}$, of $\RR^{d_k}$ centered at the origin for all $k \in [K]$. This setting is employed extensively in the subspace learning literature~\citep{Soltanolkotabi2012,Wang13-lrr-ssc,WangX13,Wang16-graphconnectivity}. Note that we still assume the columns of noisy data $\bX$ have unit $\ell^{2}$-norm for simplicity of notations. We then have the major theorem below stating the theoretical guarantee of the subspace detection property of noisy $\ell^{0}$-SSC under the semi-random model.

\newpage

Before stating this theorem, we introduce the following definition of subspace affinity, which is widely used in the analysis of semi-random model in the sparse subspace clustering literature.

\begin{definition}
\label{def::subspaces-affinity}
(\textup{Subspace affinity})
The affinity between two subspaces, $\cS_k$ and $\cS_l$ with $k,l \in [K]$, is defined by
\bsals
&{\rm aff}(\cS_k,\cS_l) = \sqrt{\sum\limits_{t=1}^{\min\{k,l\}} \cos^2 \theta_{kl}^{(t)}},
\esals%
where $\cos \theta_{kl}^{(t)}$ is the $t$-th canonical angle between $\cS_k$ and $\cS_l$ defined in~\citet{Soltanolkotabi2012}. Let $\bU^{(k)}$ and $\bU^{(l)}$ be the orthonormal basis for $\cS_k$ and $\cS_l$ respectively, then
\begin{small}\begin{align*}
&\cos \theta_{kl}^{(t)} = \sup_{\bu \in \cS_k, \bv \in \cS_l} \frac{\bu^{\top}\bv}{\|\bu\|_2 \|\bv\|_2} = \frac{{\bu^t}^{\top} \bv^t}{\|\bu^t\|_2 \|\bv^t\|_2},
\end{align*}\end{small}%
with orthogonality:  $\bu^{\top} \bu^j=0$, $\bv^{\top} \bv^j=0$, $j = 1,\ldots,t-1$.
It can be verified that ${\rm aff}(\cS_k,\cS_l) = \|{\bU^{(k)}}^{\top}\bU^{(l)}\|_F$.
\end{definition}

\begin{theorem}\label{theorem::noisy-l0ssc-subspace-detection-lambda-random}
{\rm (Subspace detection property holds for noisy $\ell^{0}$-SSC under semi-random model with conditions in terms of $\lambda$)}
Under the semi-random model, let nonzero vector $\bbeta^*$ be an optimal solution to the noisy $\ell^{0}$-SSC problem (\ref{eq:noisy-l0ssc-i}) for point $\bx_i$ with $\|\bbeta^*\|_0=r^* > 1$, $n_k \ge d_k+1$ for every $1\le k \le K$, and there exists $1 < r_0 \le \floor{\frac{1}{\lambda}}$ such that $r^* \le r_0$. Suppose $c_1 > 0$ is an arbitrary small constant, $\eps_0, \eps_1 > 0$ are small constants, and $d_k$ is large enough such that $d_k \ge \floor{\frac{1}{\lambda}}$, $2d^{-0.05}_k + 2d^{-0.1}_k \le \eps_0$ and $\sqrt{\frac{1}{\lambda d_k}} + \sqrt{ \frac{2}{\lambda d_k} \log{\frac{en_k}{r_0}}} \le \eps_1$ hold for all $k \in [K]$. Define
\bsal\label{eq:sigma-min-clean-data}
\sigma'_{\min} \defeq \frac{1}{1+\eps_0} \pth{1-\sqrt{c_1} - \eps_1},
\esal%,
$c \defeq \sqrt{\frac{\sigma'^2_{\min} - (r_0-1){\rm aff}(\cS_{t_1}, \cS_{t_2})}{r_0}}$. For $t>0$ such that $\frac{1}{d_{\max}} - 2t \sqrt{1 - \frac 1 {d_{\max}}}-t^2 > 0$, suppose
\begin{align}\label{eq:noisy-l0ssc-sdp-random-subspaces-affnity}
\max_{t_1,t_1 \in [K], t_1 \neq t_2} {\rm aff}(\cS_{t_1}, \cS_{t_2}) &< \frac{\sigma'^2_{\min}}{r_0-1},\\\label{eq:noisy-l0ssc-sdp-random-cond1}
\hspace{0.8in}\delta &< c,\\\label{eq:noisy-l0ssc-sdp-random-cond2}
\hspace{0in}\delta + \frac{2\delta}{\sqrt{r_0} (c-\delta) } &\le \frac{1}{d_{\max}} - 2t\sqrt{1 - \frac 1 {d_{\max}}}-t^2,\\\label{eq:noisy-l0ssc-sdp-random-cond3}
\hspace{0.3in}\frac{\delta}{c-\delta} + \frac{2\delta}{\sqrt{r_0} (c-\delta)} &< 1,\\\label{eq:noisy-l0ssc-sdp-lambda-random}
\hspace{0.8in} \lambda'_{0} &< \lambda < 1,
\end{align}
where $\lambda'_0 \defeq \max\{\lambda'_1,\lambda'_2\}$,
\bals
&\lambda'_{1} \defeq \inf\{0 < \lambda < 1 \colon \sqrt{1-\lambda} + \frac{2\delta}{\sqrt{r_0} (c-\delta) \sqrt{\lambda}}
< \frac{1}{d_{\max}} - 2t\sqrt{1 - \frac 1 {d_{\max}}}-t^2 -\delta\},
\\
&\lambda'_{2} \defeq \inf\{0 < \lambda < 1 \colon \lambda - \frac{2\delta}{\sqrt{r_0} (c-\delta)} \frac{1}{\sqrt{\lambda}} > \frac{\delta}{c-\delta}\}.
\eals%%
When the conditions in Lemma~\ref{lemma::point-to-subspace-concentration} hold for all $k \in [K]$ and every point $\by \in \bY^{(k)}$, then with probability at least $1-\sum_{k=1}^K \pth{\exp(-c_1 d_k)+2n_k\exp\pth{-d^{0.9}_k}} -8\sum\limits_{k=1}^K n_k\exp(-\frac{d_k t^2}{2})$, the subspace detection property holds for $\bx_i$ with $\bbeta^*$ for all $i \in [n]$.
\end{theorem}

\begin{remark}[\textup{Advantage of Noisy $\ell^{0}$-SSC in terms of Subspace Affinity}]
\label{remark::subspace-affinity}
It is well-known that the difficulty of achieving the subspace detection property increases with larger affinity between subspaces, that is, the subspace are closer to each other.
Our analysis reveals the significant advantage of noisy $\ell^{0}$-SSC over $\ell^1$-SSC in terms of the maximum subspace affinity. To the best of our knowledge, the best theoretical result of $\ell^1$-SSC, including its geometrical analysis~\citep{Soltanolkotabi2012} and the subsequent works on noisy or dimensionality-reduced data~\citep{WangX13,Wang2015-dr-l1ssc}, requires that the maximum subspace affinity satisfies
\bsal\label{eq:subspace-affinity-l1ssc}
&\max_{t_1,t_1 \in [K], t_1 \neq t_2} {\rm aff}(\cS_{t_1}, \cS_{t_2}) < \sqrt{d_0} \cdot \frac{\bar c \sqrt{\log{\rho}}}{8{\sqrt 2} \log{n}},
\esal
under the setting in~\citet{Soltanolkotabi2012} that $d_k = d_0$ and $n_k = \rho d_0 + 1$ for all $k \in [K]$, so that $n = K (\rho d_0 + 1)$. When $n  > \exp\pth{d_0^\tau}$ for $\tau \in (0.5,0.9)$, then (\ref{eq:subspace-affinity-l1ssc}) requires $\max_{t_1,t_1 \in [K], t_1 \neq t_2} {\rm aff}(\cS_{t_1}, \cS_{t_2}) \to 0$ when $d_0 \to \infty$, while the condition of noisy $\ell^{0}$-SSC, (\ref{eq:noisy-l0ssc-sdp-random-subspaces-affnity}), only requires that $\max_{t_1,t_1 \in [K], t_1 \neq t_2} {\rm aff}(\cS_{t_1}, \cS_{t_2})  < \frac{\sigma'^2_{\min}}{r_0-1}$ when $d_0$ is sufficiently large. Such less restive condition on the maximum subspace affinity reveals the theoretical advantage of noisy $\ell^{0}$-SSC over $\ell^1$-SSC.
\end{remark}

\section{Noisy-DR-$\ell^{0}$-SSC: Noisy $\ell^{0}$-SSC on Dimensionality Reduced Data}
\label{sec::noisy-dr-l0ssc}
Albeit the theoretical guarantee and compelling empirical performance of noisy $\ell^{0}$-SSC to be shown in the experimental results, the computational cost of noisy $\ell^{0}$-SSC is high with the high dimensionality of the data. In this section, we propose Noisy Dimensionality Reduced $\ell^{0}$-SSC (Noisy-DR-$\ell^{0}$-SSC) which performs noisy $\ell^{0}$-SSC on dimensionality reduced data. The theoretical guarantee on the correctness of Noisy-DR-$\ell^{0}$-SSC under the deterministic model as well as its empirical performance are presented.

\subsection{Method}
\label{sec::noisy-dr-l0ssc-method}
Noisy-DR-$\ell^{0}$-SSC performs subspace clustering by the following two steps: 1) obtain the dimension reduced data $\tilde \bX = \bP \bX$ with a linear transformation $\bP \in \RR^{p \times d}$ ($p < d$). 2) perform noisy $\ell^{0}$-SSC on the compressed data $\tilde \bX$:
\bal\label{eq:noisy-dr-l0ssc-i}
&\mathop {\min }\limits_{{\tilde \bbeta \in \RR^n, \tilde \bbeta_i = 0}} L(\tilde \bbeta) = {\|\tilde \bx_i - \tilde \bX \bbeta\|_2^2 + {\tilde \lambda}\|{\tilde \bbeta}\|_0}.
\eal%%
If $p < d$, Noisy-DR-$\ell^{0}$-SSC operates on the compressed data $\tilde \bX$ rather than on the original data, so that the efficiency is improved. We  introduce two  types of random projection for Noisy-DR-$\ell^{0}$-SSC in the following two subsections.

\subsection{Randomized Low-Rank Approximation}
\label{sec::low-rank-approx}
High-dimensional data often exhibits low-dimensional structures, which often leads to low-rankness of the data matrix. Intuitively, if the data is low rank, then it could be safe to perform noisy $\ell^{0}$-SSC on its dimensionality reduced version by the linear projection $\bP$, and it is expected that $\bP$ can preserve the information of the subspaces contained in the original data as much as possible, while effectively removing uninformative dimensions. To this end, we propose to choose $\bP$ as a random projection induced by randomized low-rank approximation of the data.

\newpage

The  merit of random projection (RP) is highlighted by the celebrated Johnson-Lindenstrauss Lemma~\citep{Article:JL84}.  In the past 20 years or more, RP has been used extensively in dimension reduction, approximate near neighbor search, compressed sensing, computational biology, etc~\citep{Proc:Dasgupta_UAI00,Proc:Bingham_KDD01,Article:Buher_01,Article:Achlioptas_JCSS03,Proc:Fern_ICML03,Proc:Datar_SCG04,Article:Candes_IT06,Article:Donoho_IT06,Proc:Frund_NIPS07,li2007very,li2017simple,li2019sign}. In particular,  RP has been employed to accelerate  numerical matrix computation and matrix optimization problems, including matrix decomposition~\citep{Frieze2004-fast-monto-carlo-lowrank,Drineas2004-large-graph-svd,Sarlos2006-large-matrix-random-projection,
Drineas2006-fast-monto-carlo-lowrank,Drineas2008-matrix-decomposition,Mahoney2009-matrix-decomposition,
Drineas2011-fast-least-square-approximation,Lu2013-fast-ridge-regression-random-subsample}.

Formally, a random matrix $\bT \in \RR^{n \times p}$ is generated such that each element $\bT_{ij}$ is sampled independently according to the Gaussian distribution $\cN(0,1)$. QR decomposition is then performed on $\bX \bT$ to obtain the basis of its column space, namely $\bX \bT = \bQ \bR$ where $\bQ \in \RR^{d \times p}$ is an orthogonal matrix of rank $p$ and $\bR \in \RR^{p \times p}$ is an upper triangle matrix. The columns of $\bQ$ form the orthogonal basis for the sample matrix $\bX \bT$. An approximation of $\bX$ is then obtained by projecting $\bX$ onto the column space of $\bX \bT$: $\bQ\bQ^{\top}\bX = \bQ\bW = \hat \bX$ where $\bW = \bQ^{\top}\bX \in \RR^{p \times n}$. In this manner, a randomized low-rank decomposition of $\bX$ is achieved by
\bals
&\hat \bX = \bQ\bW.
\eals%%
It is proved that the low rank approximation $\bar {\bX}$ is close to $\bX$ in spectral norm~\citep{Halko2011-random-matrix-decomposition}. We present probabilistic result in Theorem~\ref{theorem::noisy-dr-l0ssc-subspace-detection} on the correctness of Noisy-DR-$\ell^{0}$-SSC using the random projection induced by randomized low-rank decomposition of the data $\bX$, namely $\bP = \bQ^{\top}$. In the sequel, $\tilde \bx = \bP \bx$ for any $\bx \in \RR^n$. To guarantee the subspace detection property on the dimensionality-reduced data $\tilde \bX$, it is crucial to ensure that the conditions, such as (\ref{eq:noisy-l0ssc-sdp-M}) (\ref{eq:noisy-l0ssc-sdp-mu}) in Theorem~\ref{theorem::noisy-l0ssc-subspace-detection-lambda}, still hold after  linear transformation.

Each subspace $\cS_k$ is transformed into $\tilde \cS_k = \bP (\cS_k)$ with dimension $\tilde d_k$. We denote by $\tilde \bbeta^*$ an optimal solution to (\ref{eq:noisy-dr-l0ssc-i}), and define $C_{p,p_0} \defeq \big(1+17\sqrt{1+\frac{p_0}{p-p_0}}\big) \sigma_{p_0+1} + \frac{8\sqrt{p}}{p-p_0+1} (\sum\limits_{j > p_0} \sigma_j^2)^{\frac{1}{2}}$ with $p_0 \ge 2$. We also define the following quantities for the convenience of our analysis, which correspond to $M_i$, ${\bar \sigma}_{\bY,r}$, $\sigma_{\bX,r}$ and $\mu_{r}$ used in the analysis on the original data:
\bal\label{eq:tilde-Mi}
&\tilde M_i \defeq \min\{d(\tilde \by_i, \bH) \colon \bH \in \cH_{\tilde \by_i, \tilde d_k}\},
\eal%%
where $\cH_{\tilde \by_i, \tilde d_k}$ is all the external subspaces of $\tilde \by_i$ with dimension no greater than $\tilde d_k$ in the transformed space~by~$\bP$,
\bal
{\bar \sigma}_{\tilde \bY,r} &\defeq \min_{\bbeta: \|\bbeta\|_0 = r, {\rm rank}(\tilde \bY_{\bbeta}) = \|\bbeta\|_0} \sigma_{\min}(\tilde \bY_{\bbeta}),\label{eq:tilde-bar-sigma-Y-r}
 \\
\sigma_{\tilde \bX,r} &\defeq \min\{\sigma_{\min}(\tilde \bX_{\bbeta}) \colon 1 \le \|\bbeta\|_0 \le r\}, \label{eq:tilde-sigma-X-r} \\
\tilde \mu_{r} &\defeq \frac{\delta} {\min_{1 \le r' < r}  {\bar \sigma}_{\tilde \bY,r} - \delta}. \label{eq:tilde-mu-r}
\eal%%

\newpage

\begin{theorem}\label{theorem::noisy-dr-l0ssc-subspace-detection}
{\rm (Subspace detection property holds for Noisy-DR-$\ell^{0}$-SSC under deterministic model)}
For point $\by_i \in \cS_k$ with $n_k \ge d_k+1$, let nonzero vector $\bbeta^*$ be an optimal solution to the noisy $\ell^{0}$-SSC problem (\ref{eq:noisy-l0ssc-i}), and $\tilde \bbeta^*$ be an optimal solution to (\ref{eq:noisy-dr-l0ssc-i}) with $\bP$ being the CSP described in the beginning of this subsection. Suppose there exists $1 < r_0 \le \floor{\frac 1\lambda}$ such that $1 < \norm{\bbeta^*}{0} \le r_0$ and $1 < \norm{\tilde \bbeta^*}{0} \le r_0$. Suppose $\bY^{(k)}$ is in general position, and $\tilde M_{i,\delta} \defeq \tilde M_i - \delta$. Furthermore, suppose the following conditions hold:
\begin{enumerate}%[leftmargin=*]
\item[(i)]

$C_{p,p_0} + 2\delta \sqrt{\tilde d_k} < \sigma^{(k)}_{\bY}$,
where $\sigma^{(k)}_{\bY} \defeq \min \{\sigma_{\min}(\bA) \colon \bA \subseteq \bY^{(k)}, \bA \in \RR^{d \times n'}, n' \le \tilde d_k\}$;

\item[(ii)] 
$\min_{1\le r \le \tilde d_k} \sigma_{\bY, r} >  C_{p,p_0} - 2\delta \sqrt{\tilde d_{k}}$ and
\bals
&M_i - \pth{C_{p,p_0} + \delta + \frac{\pth{C_{p,p_0} + 2\delta \sqrt{\tilde d_k}}(1+\delta)}{\min_{1\le r \le \tilde d_k} \sigma_{\bY, r} - C_{p,p_0} - 2\delta \sqrt{\tilde d_k}}  } > \delta + \frac{2\delta}{\sigma_{\bX,r_0} - C_{p,p_0}};
\eals%%

\item[(iii)]
$3\delta <  \min_{1 \le r < r_0} \bar \sigma_{\bY,r} - C_{p,p_0}$, and
\bals
&\frac{\delta} {\min_{1 \le r < r_0}  \bar \sigma_{\bY,r_0} - C_{p,p_0} - 3\delta}
< 1-\frac{2\delta}{\sigma_{\bX,r_0} - C_{p,p_0}}.
\eals%%

\end{enumerate}

If $\tilde \lambda_{0} < \tilde \lambda < 1$,
where $\tilde \lambda_0 = \max\{\tilde \lambda_1, \tilde \lambda_2\}$ and
\bal
&\tilde \lambda_{1} = \inf\{0 < \tilde \lambda < 1 \colon \sqrt{1- \tilde \lambda} + \frac{2\delta}{ \sigma_{\tilde \bX,r_0} \sqrt{\tilde \lambda}}< \tilde M_{i,\delta}\}, \label{eq:noisy-dr-l0ssc-sdp-min-lambda-1}
\\
&\tilde \lambda_{2} = \inf\{0 < \tilde \lambda < 1 \colon \tilde \lambda - \frac{2\delta}{ \sigma_{\tilde \bX,r_0}} \frac{1}{\sqrt{\tilde \lambda}} > \tilde \mu_{r_0}\}, \label{eq:noisy-dr-l0ssc-sdp-min-lambda-2}
\eal%%
then with probability at least $1-6e^{-p}$, the subspace detection property holds for $\tilde \bx_i$ with $\tilde \bbeta^*$. Here $\tilde M_i$, $\tilde \mu_r$ and $\tilde \sigma_{\tilde \bX,r_0}$ are defined in (\ref{eq:tilde-Mi}), (\ref{eq:tilde-mu-r}) and (\ref{eq:tilde-sigma-X-r}) respectively.
\end{theorem}

\subsection{Very Sparse Random Projections}
\label{sec::osnap}
In this subsection, we study the case when the linear transformation $\bP$ for the dimensionality reduced $\ell^0$-SSC problem (\ref{eq:noisy-dr-l0ssc-i}) is a sparse matrix~\citep{Article:Charikar_2004,Article:Cormode_05,li2007very,Proc:Weinberger_ICML2009,Article:Gilbert_IEEE10,Proc:Li_NIPS11_Hashing,NelsonN13-OSNAP,li2022gcwsnet}. In particular, we choose $\bP$ such that each column of $\bP$ only has $1$ nonzero element, in a fashion known as ``count-sketch''~\citep{Article:Charikar_2004}. \cite{Proc:Weinberger_ICML2009} applied count-sketch as a dimension reduction tool for machine learning. The work of~\citep{Proc:Li_NIPS11_Hashing}, in addition to developing hash learning algorithm based on minwise hashing, also provided the thorough theoretical analysis for count-sketch in the context of estimating inner products. The conclusion from~\citet{Proc:Li_NIPS11_Hashing}  is that,  to estimate inner products, we should use count-sketch (or very sparse random projections~\citep{li2007very})  instead of the original (dense) random projections, because count-sketch is not only computationally much more efficient but also (slightly) more accurate, as far as the task of similarity estimation is concerned.

Using those nice theoretical properties of count-sketch projections, we have the following theorem about the correctness of Noisy-DR-$\ell^{0}$-SSC when $\bP$ has only  $1$ nonzero element in each column. For brevity, we name such a projection matrix to be ``CSP".

\begin{theorem}\label{theorem::noisy-dr-l0ssc-subspace-detection-csp}
{\rm (Subspace detection property holds for Noisy-DR-$\ell^{0}$-SSC under deterministic model with $\bP$ being the CSP)}
Let $r' < \min\set{d,n}$ be the rank of the clean data matrix $\bY$.
For point $\by_i \in \cS_k$  with $n_k \ge d_k+1$, let nonzero vector $\bbeta^*$ be an optimal solution to the noisy $\ell^{0}$-SSC problem (\ref{eq:noisy-l0ssc-i}), and $\tilde \bbeta^*$ be an optimal solution to (\ref{eq:noisy-dr-l0ssc-i}) with $\bP$ being the CSP described in the beginning of this subsection. Suppose there exists $1 < r_0 \le \floor{\frac 1\lambda}$ such that $1 < \norm{\bbeta^*}{0} \le r_0$ and $1 < \norm{\tilde \bbeta^*}{0} \le r_0$. Suppose $\bY$ is in general position, $\by_i \in \cS_k$ for some $1 \le k \le K$, $\delta < \min_{1 \le r < r_0} \bar \sigma_{\bY,r}$. Let $M_{i,\delta} \defeq M_i - \delta$, $\varepsilon$ be a positive number such that $0 < \varepsilon \le 1$. Further suppose $\eps (1+\delta) < \sigma^{(k)}_{\bY}$, where $\sigma^{(k)}_{\bY}$ is defined in Theorem~\ref{theorem::noisy-dr-l0ssc-subspace-detection}, and the following conditions hold:
\vspace{-.07in}
\begin{flalign}\label{eq:noisy-dr-l0ssc-sdp-M}
(1-\eps)M_i &>  \frac{2 (1+\varepsilon) \delta}{\sigma_{\bX,r_0} - \pth{\delta \pth{2 \sqrt{r_0} + \eps(\sqrt{r_0}+1)} + \eps}}, \\
\frac{\delta} {\min_{1 \le r' < r_0} \bar \sigma_{\bY,r} - \eps(1+\delta) -  \delta} &< 1-\frac{2(1+\varepsilon)\delta}{\sigma_{\bX,r_0} - \pth{\delta \pth{2 \sqrt{r_0} + \eps(\sqrt{r_0}+1)} + \eps}}.
\end{flalign}%%
Then if $\tilde \lambda_{0} < \tilde \lambda < 1$, where $\tilde \lambda_0 \defeq \max\{\lambda_1,\lambda_2\}$ and
\bal
&\lambda_{1} \defeq \inf\{0 < \tilde \lambda < 1 \colon \sqrt{1+\varepsilon-\tilde \lambda} + \frac{2\delta}{\sigma_{\tilde \bX,r_0} \sqrt{\tilde \lambda}}< \tilde M_{i,\delta}\}, \label{eq:noisy-dr-l0ssc-sdp-min-lambda-1}
\\
&\lambda_{2} \defeq \inf\{0 < \lambda < 1 \colon \lambda - \frac{2\delta}{\sigma_{\tilde \bX,r_0}} \frac{1}{\sqrt{\lambda}} > \tilde \mu_{r_0}\}, \label{eq:noisy-dr-l0ssc-sdp-min-lambda-2}
\eal%%
then with probability at least $1-\delta'$ for all $\delta' \in (0,1)$, the subspace detection property holds for $\tilde \bx_i$ with $\tilde \bbeta^*$ and $p = \cO(\frac{r'^2}{\eps^2 \delta'})$. Here $\tilde \mu_{r_0}$ and $\sigma_{\tilde \bX,r_0}$ are defined in (\ref{eq:tilde-mu-r}) and (\ref{eq:tilde-sigma-X-r}) respectively.
\end{theorem}

\subsection{The Algorithm of Noisy-DR-$\ell^{0}$-SSC}
We denote by Noisy-DR-$\ell^{0}$-SSC-LR the Noisy-DR-$\ell^{0}$-SSC with random projection induced by randomized low-rank approximation in Section~\ref{sec::low-rank-approx}, and denote by Noisy-DR-$\ell^{0}$-SSC-CSP the Noisy-DR-$\ell^{0}$-SSC with CSP serving as the random projection in Section~\ref{sec::osnap}.

\begin{algorithm}[h]
\renewcommand{\algorithmicrequire}{\textbf{Input:}}
\renewcommand\algorithmicensure {\textbf{Output:} }
\small
\caption{\small Noisy Dimensionality Reduced $\ell^{0}$-Sparse Subspace Clustering by Randomized Low-Rank Approximation (Noisy-DR-$\ell^{0}$-SSC-LR)}
\label{alg:noisy-dr-l0ssc-lr}
\begin{algorithmic}[1]
\STATE{Generate a Gaussian random matrix $\bT \in \RR^{n \times p}$ where each element $\bT_{ij}$ is sampled independently according to the standard Gaussian distribution $\cN(0,1)$}
\STATE{Perform QR decomposition on $\bX \bT$, $\bX \bT = \bQ \bR$ where $\bQ \in \RR^{d \times p}$}
\STATE{Set the linear transformation $\bP = \bQ^{\top}$, and obtain the dimensionality reduced data $\tilde \bX = \bP \bX$}
\STATE{Perform noisy $\ell^{0}$-SSC on $\tilde \bX$ using Algorithm~\ref{alg:PGD-l0ssc}}
\end{algorithmic}
\end{algorithm}
Noisy-DR-$\ell^{0}$-SSC-LR is described by Algorithm~\ref{alg:noisy-dr-l0ssc-lr}. The algorithm of Noisy-DR-$\ell^{0}$-SSC-CSP is similar to Algorithm~\ref{alg:noisy-dr-l0ssc-lr} except that CSP serves as the random projection $\bP$. Algorithm~\ref{alg:PGD-l0ssc} in Section~\ref{sec::pgd-noisy-l0ssc} describes how to solve the noisy $\ell^0$-SSC problem~(\ref{eq:noisy-l0ssc-i}).

\newpage

\subsection{Time Complexity of noisy $\ell^{0}$-SSC, Noisy-DR-$\ell^{0}$-SSC-LR, Noisy-DR-$\ell^{0}$-SSC-CSP}
\label{sec::time-complexity}
The time complexity of running PGD by Algorithm~\ref{alg:PGD-l0ssc} for noisy $\ell^{0}$-SSC is $\cO(Tnd)$, where $T$ is the maximum iteration number. The time complexity of running Algorithm~\ref{alg:noisy-dr-l0ssc-lr}
for Noisy-DR-$\ell^{0}$-SSC-LR is comprised of two parts. The first part is the time complexity of steps 1-3 with matrix multiplication and QR decomposition, which is $\cO(dp^2 + pdn)$. The second part is the time complexity of step 4, which is $\cO(Tnp)$. The overall time complexity of Noisy-DR-$\ell^{0}$-SSC is $\cO(dp^2 + pdn + Tnp)$. In practice, $p$ is much smaller than $\min\set{d,n,T}$, so Noisy-DR-$\ell^{0}$-SSC-LR is more efficient than noisy $\ell^{0}$-SSC. Noisy-DR-$\ell^{0}$-SSC-CSP is even more efficient than both noisy $\ell^{0}$-SSC and Noisy-DR-$\ell^{0}$-SSC, whose time complexity is $\cO(pdn + Tnp)$. This is because the linear transformation $\bP$ obtained by CSP does require QR decomposition.

\begin{algorithm}[t]
\renewcommand{\algorithmicrequire}{\textbf{Input:}}
\renewcommand\algorithmicensure {\textbf{Output:} }
\small
\caption{Proximal Gradient Descent (PGD) for noisy $\ell^0$-SSC problem (\ref{eq:noisy-l0ssc-i})}
\label{alg:PGD-l0ssc}
\begin{algorithmic}[1]
\REQUIRE ~~\\
The initialization ${\bbeta}^{(0)}$, step size $s >0$, parameter $\lambda$, maximum iteration number $T$, stopping threshold $\varepsilon$.
\FOR{$1 \le i \le n$}
\STATE{$\tilde \bbeta^{(t)} = {\bbeta}^{(t-1)} - s \nabla g({\bbeta}^{(t-1)})$}
\STATE{${\bbeta}^{(t)} = T_{\sqrt{{2\lambda s}}}(\tilde \bbeta^{(t)})$}
\IF{$|L(\bbeta^{(t)})-L(\bbeta^{(t-1)})| < \varepsilon$}
%\PRINT
\STATE \textbf{break}
\ENDIF
\ENDFOR

\ENSURE $\hat \bbeta$ which is the suboptimal solution to (\ref{eq:noisy-l0ssc-i})
\end{algorithmic}
\end{algorithm}

\subsection{Proximal Gradient Descent (PGD) for Noisy $\ell^{0}$-SSC}
\label{sec::pgd-noisy-l0ssc}
Algorithm~\ref{alg:noisy-dr-l0ssc-lr} describes how to perform Noisy-DR-$\ell^{0}$-SSC-LR for data clustering. Note that Noisy-DR-$\ell^{0}$-SSC performs noisy $\ell^{0}$-SSC on the dimensionality reduced data $\tilde \bX$. Proximal Gradient Descent (PGD) is employed to optimize the objective function of noisy $\ell^{0}$-SSC for every data point $\bx_i$, which is described in Algorithm~\ref{alg:PGD-l0ssc}. In the $k$-th iteration of PGD for problem (\ref{eq:noisy-l0ssc-i}), the variable $\bbeta$ is updated according to
\bals
        &{\bbeta}^{(k+1)} = T_{\sqrt{{2\lambda s}}}({\bbeta}^{(k)} - s \nabla g(\bbeta^{(k)})),
\eals%%
where $s$ is a positive step size, $g(\bbeta) = \|\bx_i - \bX {\bbeta}\|_2^2$, $T_{\theta}$ is an element-wise hard thresholding operator:
\begin{align*}
        [T_{\theta}(\bu)]_j=
        \left\{
        \begin{array}
                {r@{\quad:\quad}l}
                0 & {|\bu_j| \le \theta } \\
                {\bu_j} & {\rm otherwise}
        \end{array}
        \right., \quad 1 \le j \le n.
\end{align*}%
It is proved in~\citet{Yang2017-PGD-l0-sparse-approximation} that the sequence $\{{\bbeta}^{(k)}\}$ generated by PGD converges to a critical point of (\ref{eq:noisy-l0ssc-i}).

\begin{table*}[!hbt]
\centering
\caption{\small Clustering results on various data sets, with the best three results in bold.\vspace{-0.05in}}
\resizebox{1\linewidth}{!}{
\begin{tabular}{|c|c|c|c|c|c|c|c|c|c|c|}
  \hline
  Data Set

                              &Measure & KM    & SC     &Noisy SSC &Noisy DR-SSC       &SMCE    &SSC-OMP     &Noisy $\ell^{0}$-SSC   &Noisy-DR-$\ell^{0}$-SSC-LR &Noisy-DR-$\ell^{0}$-SSC-CSP\\\hline

  \multirow{2}{*}{COIL-20}    &AC      &0.6554 &0.4278  &0.7854    &0.7764 &0.7549  &0.3389      &\textbf{0.8472} $\pm$ 0.0031   &\textbf{0.8479} $\pm$ 0.0023 &\textbf{0.8472} $\pm$ 0.0019\\ \cline{2-11}
                              &NMI     &0.7630 &0.6217  &0.9148    &0.9219 &0.8754  &0.4853      &\textbf{0.9428} $\pm$ 0.0082   &\textbf{0.9433} $\pm$ 0.0063 &\textbf{0.9429} $\pm$ 0.0037\\ \hline
  \multirow{2}{*}{COIL-100}   &AC      &0.4996 &0.2835  &0.5275    &0.5013 &0.5639  &0.1667      &\textbf{0.7683} $\pm$ 0.0020  &\textbf{0.7039} $\pm$ 0.0087 &\textbf{0.7046} $\pm$ 0.0083\\ \cline{2-11}
                              &NMI     &0.7539 &0.5923  &0.8041    &0.8019 &0.8064  &0.3757      &\textbf{0.9182} $\pm$ 0.0096  &\textbf{0.8706} $\pm$ 0.0109 &\textbf{0.8708} $\pm$ 0.0117\\ \hline
  \multirow{2}{*}{Yale-B}     &AC      &0.0954 &0.1077  &0.7850    &0.7255 &0.3293  &0.7789      &\textbf{0.8480}  $\pm$ 0.0091 &\textbf{0.8231} $\pm$ 0.0173 &\textbf{0.8318} $\pm$ 0.0112\\ \cline{2-11}
                              &NMI     &0.1258 &0.1485  &0.7760    &0.7311 &0.3812  &0.7024      &\textbf{0.8612}  $\pm$ 0.0072  &\textbf{0.8533} $\pm$ 0.0294  &\textbf{0.8593} $\pm$ 0.0133\\ \hline

  \multirow{2}{*}{MPIE S$1$}
                           &AC      &0.1164 &0.1285  &0.5892    &0.3588 &0.1721  &0.1695      &\textbf{0.6741}$\pm$ 0.0413 &\textbf{0.6741}$\pm$ 0.0938 &\textbf{0.6744}$\pm$ 0.0662 \\ \cline{2-11}
                           &NMI     &0.5049 &0.5292  &0.7653    &0.6806 &0.5514  &0.3395      &\textbf{0.8622}$\pm$ 0.0533 &\textbf{0.8622}$\pm$ 0.0834 &\textbf{0.8548}$\pm$ 0.0931\\ \hline

  \multirow{2}{*}{MPIE S$2$}
                           &AC      &0.1315 &0.1410  &0.6994    &0.4611 &0.1898  &0.2093      &\textbf{0.7527}$\pm$ 0.0115 &\textbf{0.7533}$\pm$ 0.0596 &\textbf{0.7517}$\pm$ 0.0813 \\ \cline{2-11}
                           &NMI     &0.4834 &0.5128  &0.8149    &0.7086 &0.5293  &0.4292      &\textbf{0.8939}$\pm$ 0.0389 &\textbf{0.8926} $\pm$ 0.0742 &\textbf{0.8910} $\pm$ 0.0454\\ \hline

  \multirow{2}{*}{MPIE S$3$}
                           &AC      &0.1291 &0.1459  &0.6316    &0.4841 &0.1856  &0.1787      &\textbf{0.7050}$\pm$ 0.0277 &\textbf{0.7123}$\pm$ 0.0812 &\textbf{0.7184}$\pm$ 0.1045 \\ \cline{2-11}
                           &NMI     &0.4811 &0.5185  &0.7858    &0.7340 &0.5155  &0.3415      &\textbf{0.8750}$\pm$ 0.0157 &\textbf{0.8455}$\pm$ 0.0693 &\textbf{0.8457}$\pm$ 0.0913 \\ \hline

  \multirow{2}{*}{MPIE S$4$}
                           &AC      &0.1308 &0.1463  &0.6803    &0.5511 &0.1823  &0.1680      &\textbf{0.7246}$\pm$ 0.0147 &\textbf{0.7137}$\pm$ 0.0605 &\textbf{0.7250}$\pm$ 0.0443 \\ \cline{2-11}
                           &NMI     &0.4866 &0.5280  &0.8063    &0.7955 &0.5294  &0.3345      &\textbf{0.8837}$\pm$ 0.0212 &\textbf{0.8847}$\pm$ 0.0781 &\textbf{0.8834}$\pm$ 0.0517 \\ \hline

  \multirow{2}{*}{MNIST}      &AC      &0.5236  &0.3504  &0.5714   &0.5123 &\textbf{0.6542}   &0.5561       &0.6259 $\pm$ 0.0249   &\textbf{0.6296} $\pm$ 0.1522 &\textbf{0.6310} $\pm$ 0.1031 \\ \cline{2-11}
                              &NMI     &0.4770  &0.3607   &0.6091  &0.5026  &\textbf{0.6796}  &0.5986     &\textbf{0.6501} $\pm$ 0.0196   &0.6440 $\pm$ 0.0259  &\textbf{0.6497} $\pm$ 0.0313 \\ \hline

\end{tabular}
}
\label{table:results}\vspace{-0.05in}
\end{table*}

\section{Experiments}

We demonstrate the performance of Noisy-DR-$\ell^{0}$-SSC-LR and Noisy-DR-$\ell^{0}$-SSC-CSP, with comparison to other competing clustering methods including K-means (KM), Spectral Clustering (SC), noisy SSC, Sparse Manifold Clustering and Embedding (SMCE)~\citep{ElhamifarV11} and SSC-OMP~\citep{Dyer13a} in this section. We will use Noisy-DR-$\ell^{0}$-SSC to refer to its two variants. With the coefficient matrix $\bZ$ obtained by the optimization of noisy $\ell^{0}$-SSC or Noisy-DR-$\ell^{0}$-SSC, a sparse similarity matrix is built by $\bW = \frac{|\bZ| + |\bZ^{\top}|}{2}$, and spectral clustering is performed on $\bW$ to obtain the clustering results. Two measures are used to evaluate the performance of different clustering methods, i.e. the Accuracy (AC) and the Normalized Mutual Information (NMI)~\citep{Zheng04}.

\begin{figure*}[b!]
\mbox{
\includegraphics[width=0.4\textwidth]{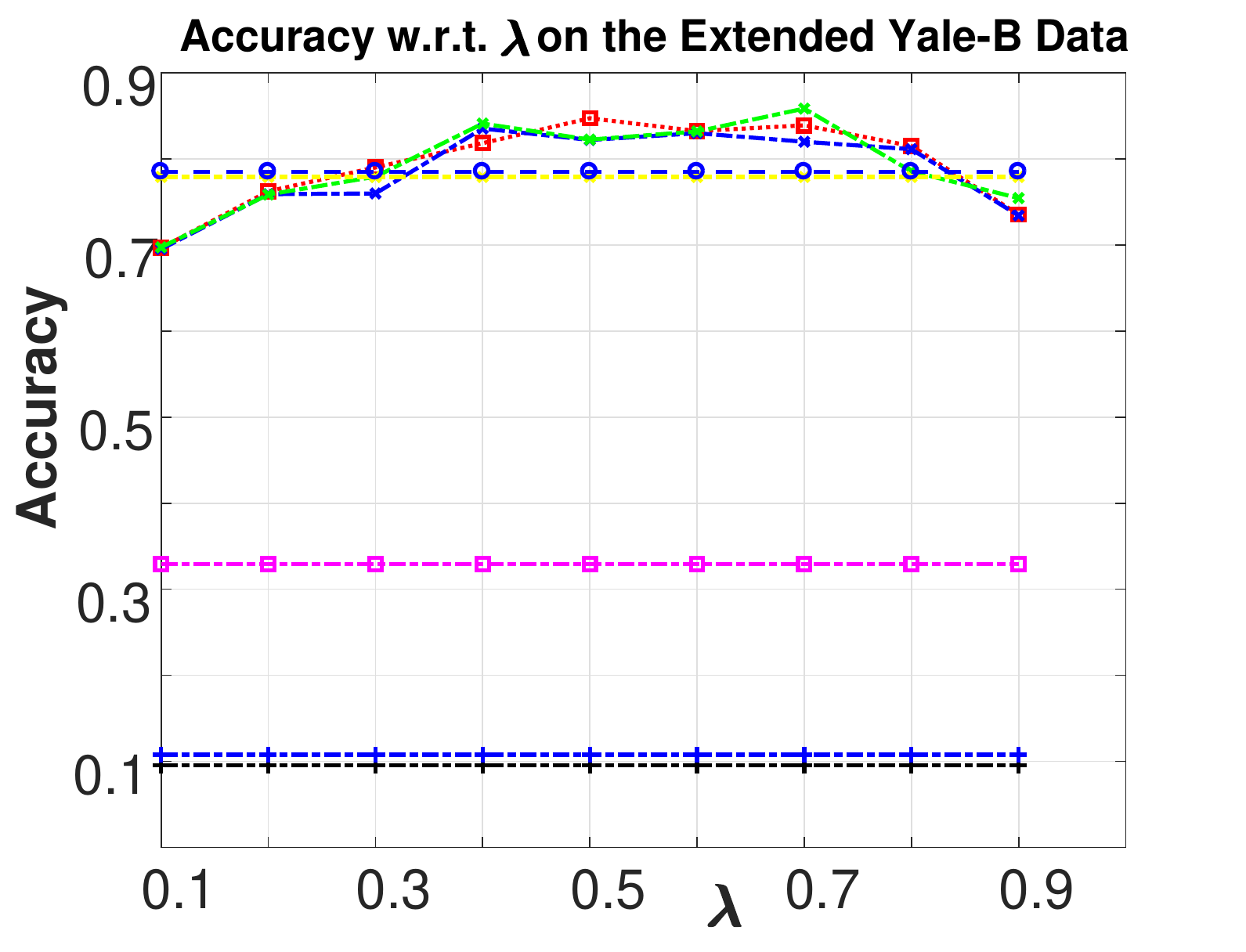}
\includegraphics[width=0.4\textwidth]{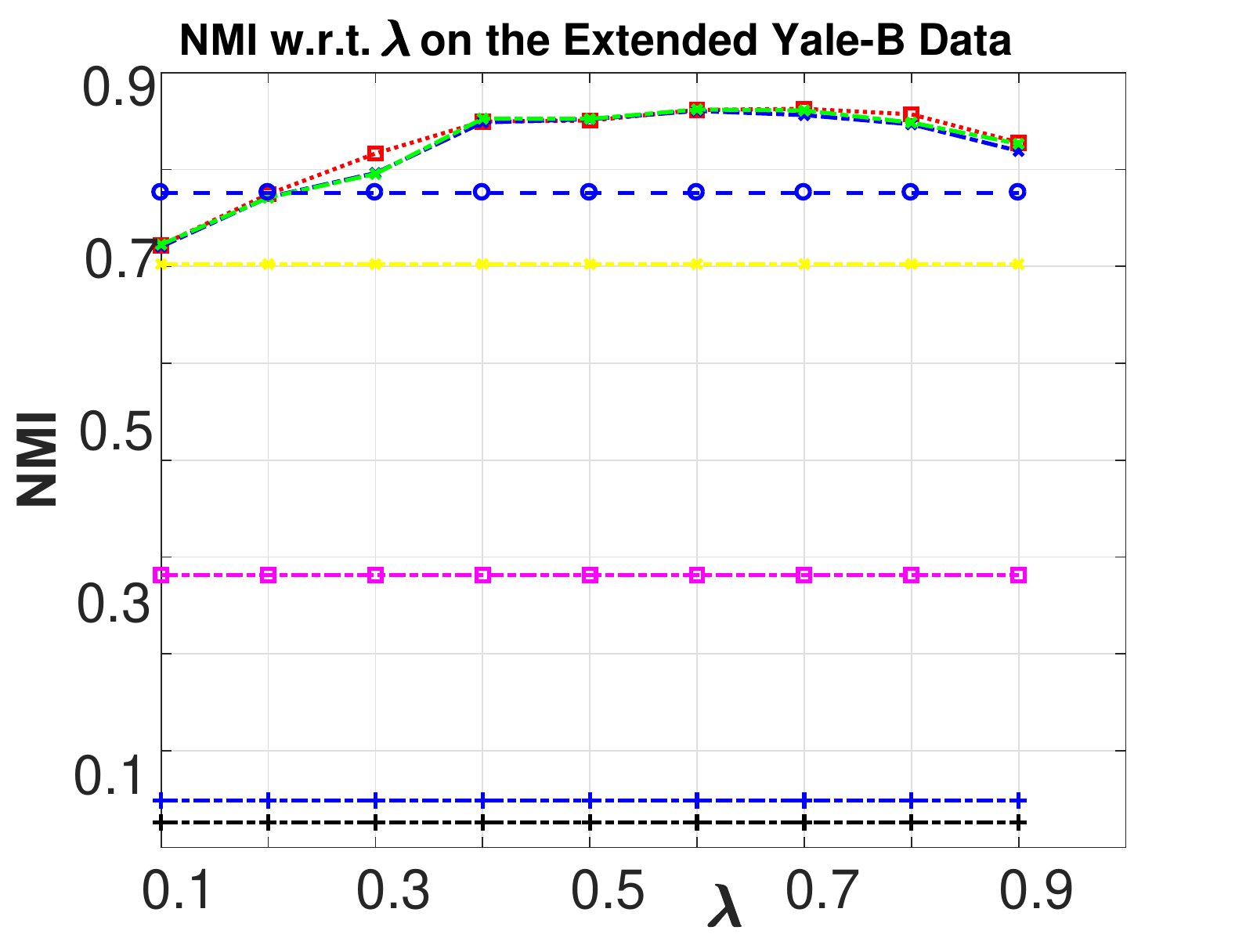}
\includegraphics[width=0.2\textwidth]{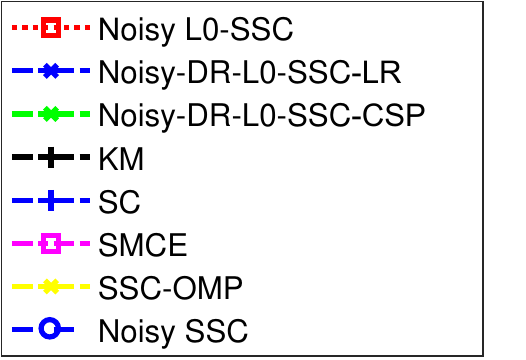}
}

\caption{Accuracy (left) and NMI (right) with respect to different values of $\lambda$ on the Extended Yale-B data set}
\label{fig:accuracy-nmi-sensitivity}\vspace{-0.25in}
\end{figure*}

We use randomized rank-$p$ decomposition of the data matrix in Noisy-DR-$\ell^{0}$-SSC-LR with $p = \frac{\min\{d,n\}}{10}$. It can be observed that noisy $\ell^{0}$-SSC and Noisy-DR-$\ell^{0}$-SSC always achieve better performance than other methods in Table~\ref{table:results}, including the noisy SSC on dimensionality reduced data (Noisy DR-SSC)~\citep{Wang2015-dr-l1ssc}. Note that noisy $\ell^{0}$-SSC has the same performance as $\ell^{0}$-SSC~\citep{YangFJYH16-L0SSC-ijcv}. Throughout all the experiments we find that the best clustering accuracy is achieved whenever $\lambda$ is chosen by $0.5 < \lambda < 0.95$, justifying our theoretical finding claimed in Remark~\ref{remark::lambda} and (\ref{eq:noisy-l0ssc-sdp-lambda}) in Theorem~\ref{theorem::noisy-l0ssc-subspace-detection-lambda}. For all the methods that involve random projection, we conduct the experiments for $30$ times and report the average performance. Note that the cluster accuracy of SSC-OMP on the extended Yale-B data set is reported according to~\citet{YouRV16-OMP}. We randomly sample $1000$ images from each class of the MNIST data set so as to collect a total number of $10000$ images on which clustering is performed, and the average performance of $10$ random sampling is reported for this data set. The actual running time of both algorithms confirms such time complexity, and we observe that Noisy-DR-$\ell^{0}$-SSC-LR is always $8.7$ times faster than noisy $\ell^{0}$-SSC with the same number of iterations, and the acceleration is boosted to $9.6$ times by Noisy-DR-$\ell^{0}$-SSC-CSP
due to sparse random projections. Figure~\ref{fig:accuracy-nmi-sensitivity} show how the accuracy and NMI varies with respect to $\lambda$ on the Extended Yale-B data set. We present more results of Noisy-DR-$\ell^{0}$-SSC-LR and Noisy-DR-$\ell^{0}$-SSC-CSP in Table~\ref{table:more-results} with different projection dimension $p$.

\vspace{0.05in}

We further demonstrate the practical implication of our theoretical analysis for noisy $\ell^{0}$-SSC. As mentioned in Remark~\ref{remark::lambda}, a relatively large $\lambda$ tends to preserve the subspace detection property. This theoretical finding is consistent with the empirical study shown in this subsection. We add Gaussian noise of zero mean and different choices of variance $\sigma^2$ to the extended Yale-B data set. Figure~\ref{fig:yaleb-noisy-level-1} to Figure~\ref{fig:yaleb-noisy-level-6} illustrate SDP violation with respect to $\lambda$ for different noise levels with $\sigma^2$ ranging over $10,20,30,40,50,60$, justifying our theoretical finding that a large $\lambda$ tends to preserve the subspace detection property for noisy $\ell^{0}$-SSC, Noisy-DR-$\ell^{0}$-SSC-LR and Noisy-DR-$\ell^{0}$-CSP. The SDP violation is defined in~\citet{WangX13} which is the percentage of pairs of data points which are mistakenly put in the same subspace by the similarity matrix $\bW$, namely the percentage of pairs $(\bx_i,\bx_j)$ with nonzero $\bW_{ij}$ while they are in fact not in the same subspace. We observe that increasing $\lambda$ effectively reduces SDP violation for noisy $\ell^{0}$-SSC, Noisy-DR-$\ell^{0}$-SSC-LR and Noisy-DR-$\ell^{0}$-CSP, confirming our theoretical prediction.

\begin{table*}[t]
\centering
\caption{\small Clustering results on various data sets, with different values of $p$ for the linear transformation $\bP$ and the best two results in bold}
\resizebox{\linewidth}{!}{
\begin{tabular}{|c|c|c|c|c|c|c|c|c|c|}
  \hline
  Data Set

                              &Measure &Noisy $\ell^{0}$-SSC   &\multicolumn{3}{|c|}{Noisy-DR-$\ell^{0}$-SSC-LR}               &\multicolumn{3}{|c|}{Noisy-DR-$\ell^{0}$-SSC-CSP} \\\hline
  $p$                           &        &           &$p=\min\{d,n\}/5$ &$p=\min\{d,n\}/10$ &$p=\min\{d,n\}/15$    &$p=\min\{d,n\}/5$ &$p=\min\{d,n\}/10$ &$p=\min\{d,n\}/15$ \\ \hline
  \multirow{2}{*}{COIL-20}    &AC      &0.8472          &{0.8479} &{0.8479} &\textbf{0.8479}                      &\textbf{0.8486} &{0.8472} &{0.8472} \\\cline{2-9}
                              &NMI     &0.9428          &0.9433   &0.9433   &\textbf{0.9433}                 &\textbf{0.9439} &{0.9428} &{0.9428}  \\ \hline
  \multirow{2}{*}{COIL-100}   &AC      &\textbf{0.7683} &{0.6992} &\textbf{0.7276} &{0.7043}                      &\textbf{0.5404} &{0.7046} &{0.7233} \\\cline{2-9}
                              &NMI     &\textbf{0.9182} &0.8626   &\textbf{0.8919} &{0.8636}                      &{0.7819} &{0.8708} &\textbf{0.8726}  \\ \hline
  \multirow{2}{*}{Yale-B}     &AC      &\textbf{0.8480}         &{0.8219} &{0.8231} &{0.8289}                      &\textbf{0.8500} &{0.8318} &{0.8277} \\\cline{2-9}
                              &NMI     &0.8612         &0.8519   &0.8527        &{0.8534}                 &{0.8538} &\textbf{0.8593} &\textbf{0.8594}  \\ \hline

\end{tabular}
}
\label{table:more-results}
\end{table*}

\begin{figure*}[htb]
    \centering % <-- added
\begin{subfigure}{0.33\textwidth}
  \includegraphics[width=\linewidth]{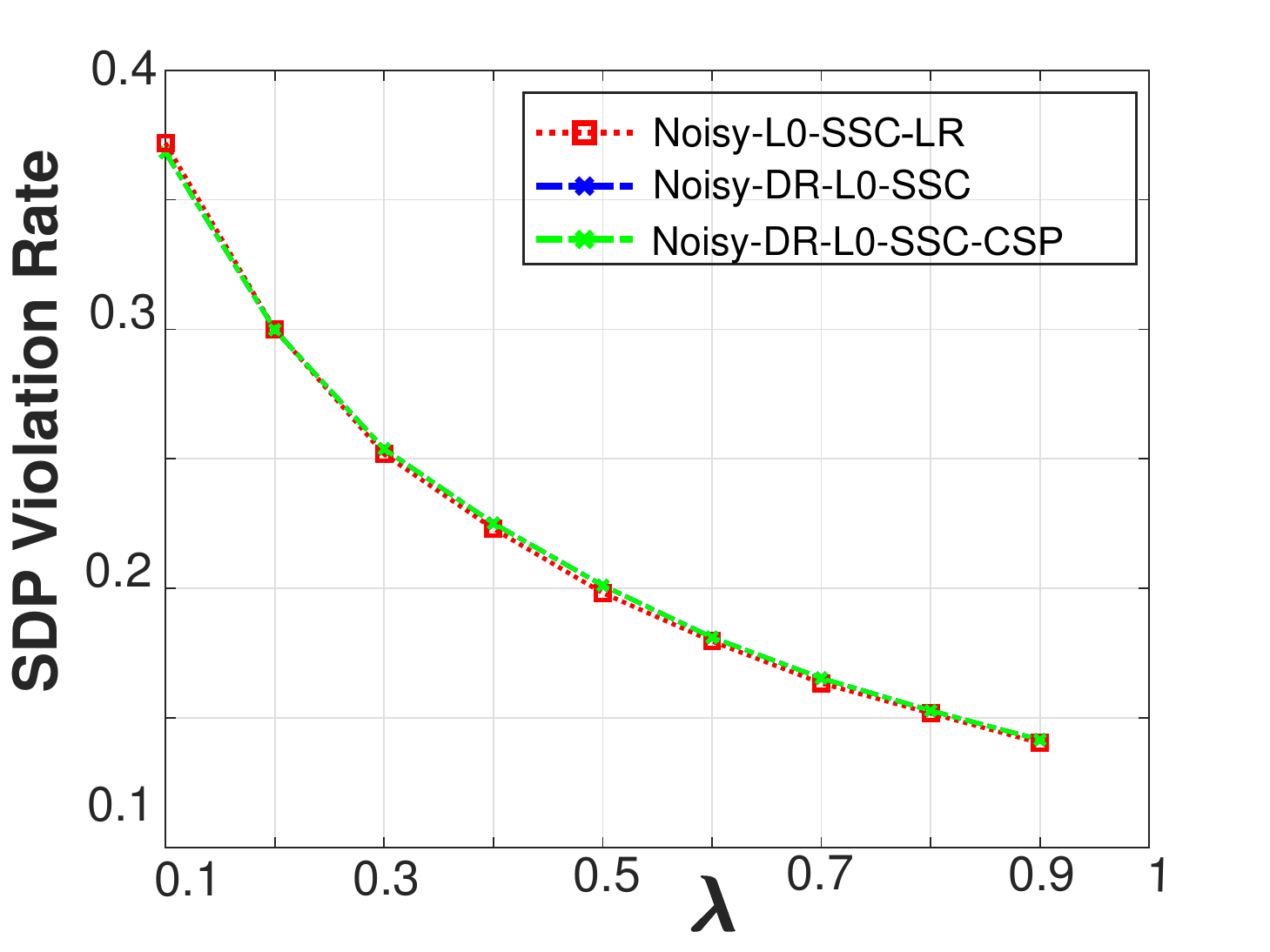}
  \caption{$\sigma^2=10$}
  \label{fig:yaleb-noisy-level-1}
\end{subfigure}\hfil % <-- added
\begin{subfigure}{0.33\textwidth}
  \includegraphics[width=\linewidth]{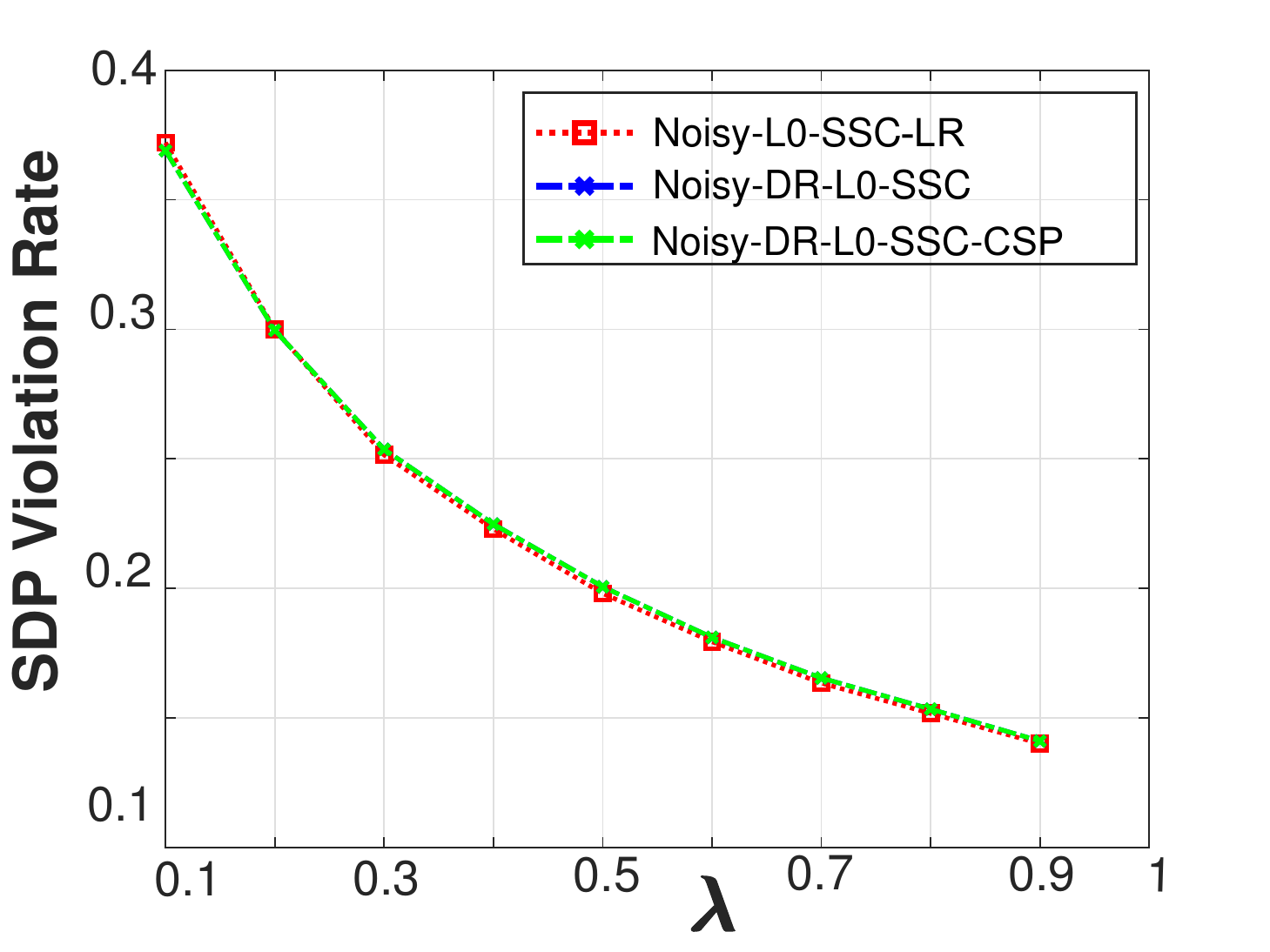}
 \caption{$\sigma^2=20$}
  \label{fig:yaleb-noisy-level-2}
\end{subfigure}\hfil % <-- added
\begin{subfigure}{0.33\textwidth}
  \includegraphics[width=\linewidth]{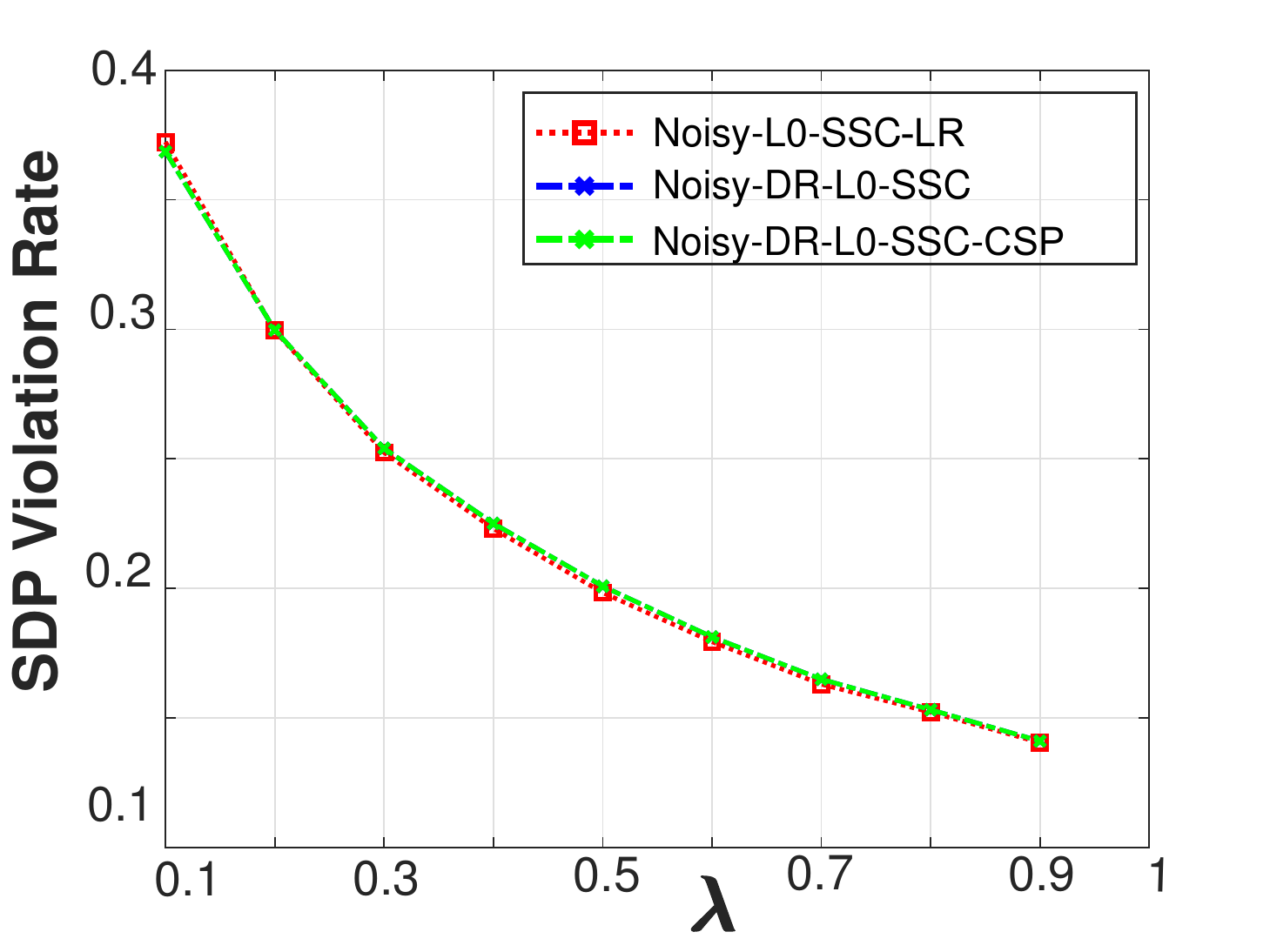}
  \caption{$\sigma^2=30$}
  \label{fig:yaleb-noisy-level-3}
\end{subfigure}

\medskip
\begin{subfigure}{0.33\textwidth}
  \includegraphics[width=\linewidth]{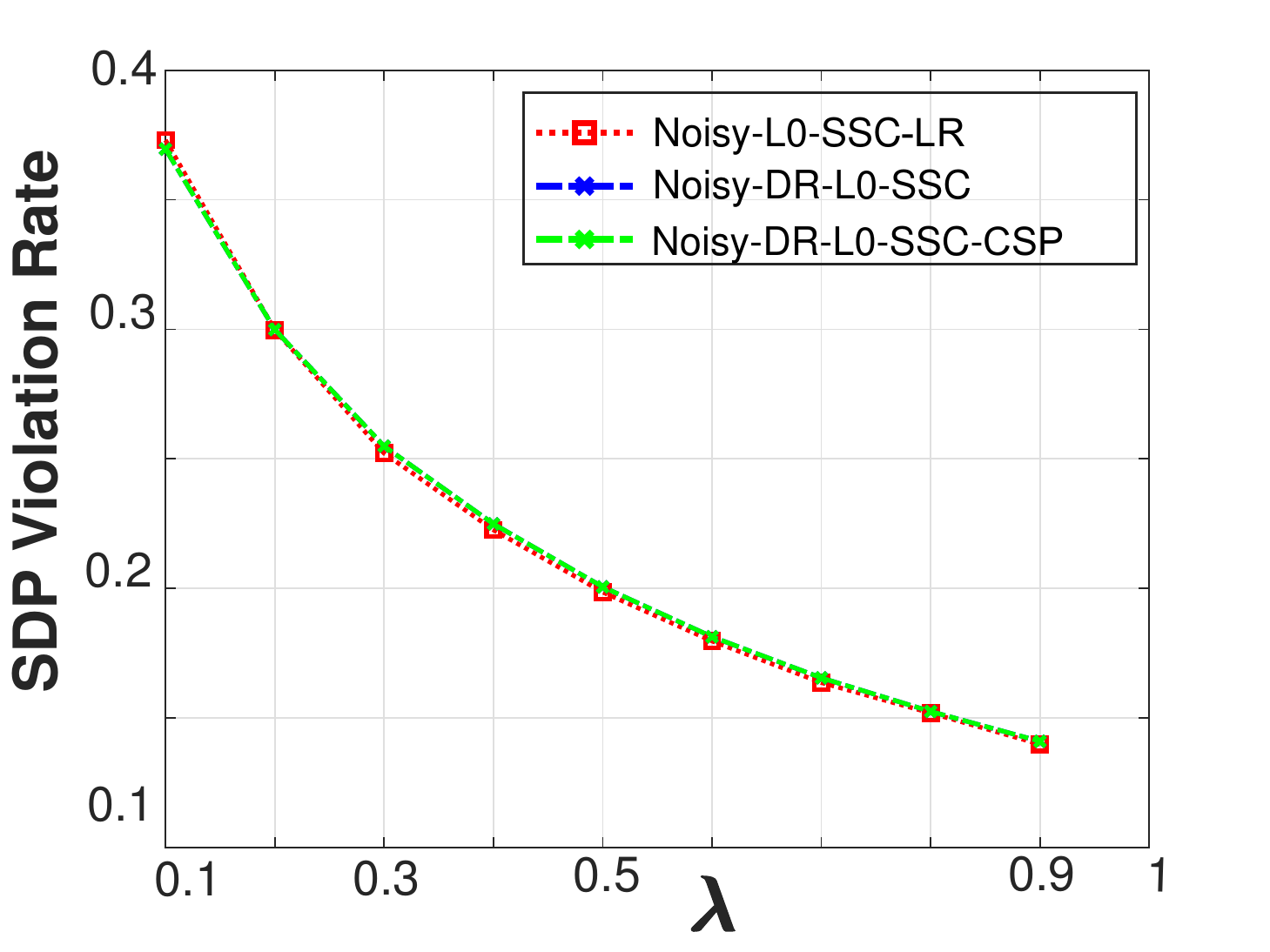}
  \caption{$\sigma^2=40$}
  \label{fig:yaleb-noisy-level-4}
\end{subfigure}\hfil % <-- added
\begin{subfigure}{0.33\textwidth}
  \includegraphics[width=\linewidth]{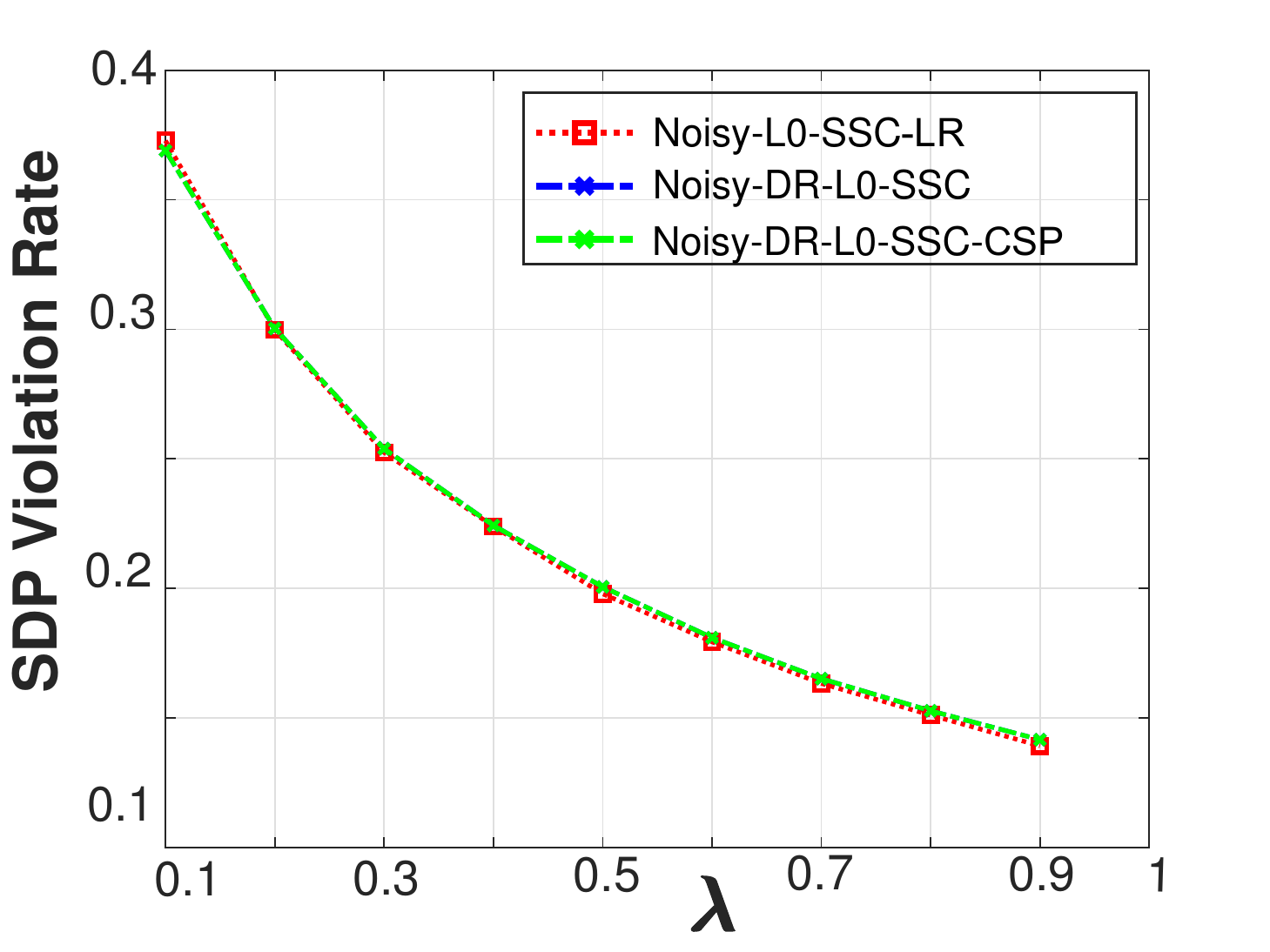}
  \caption{$\sigma^2=50$}
  \label{fig:yaleb-noisy-level-5}
\end{subfigure}\hfil % <-- added
\begin{subfigure}{0.33\textwidth}
  \includegraphics[width=\linewidth]{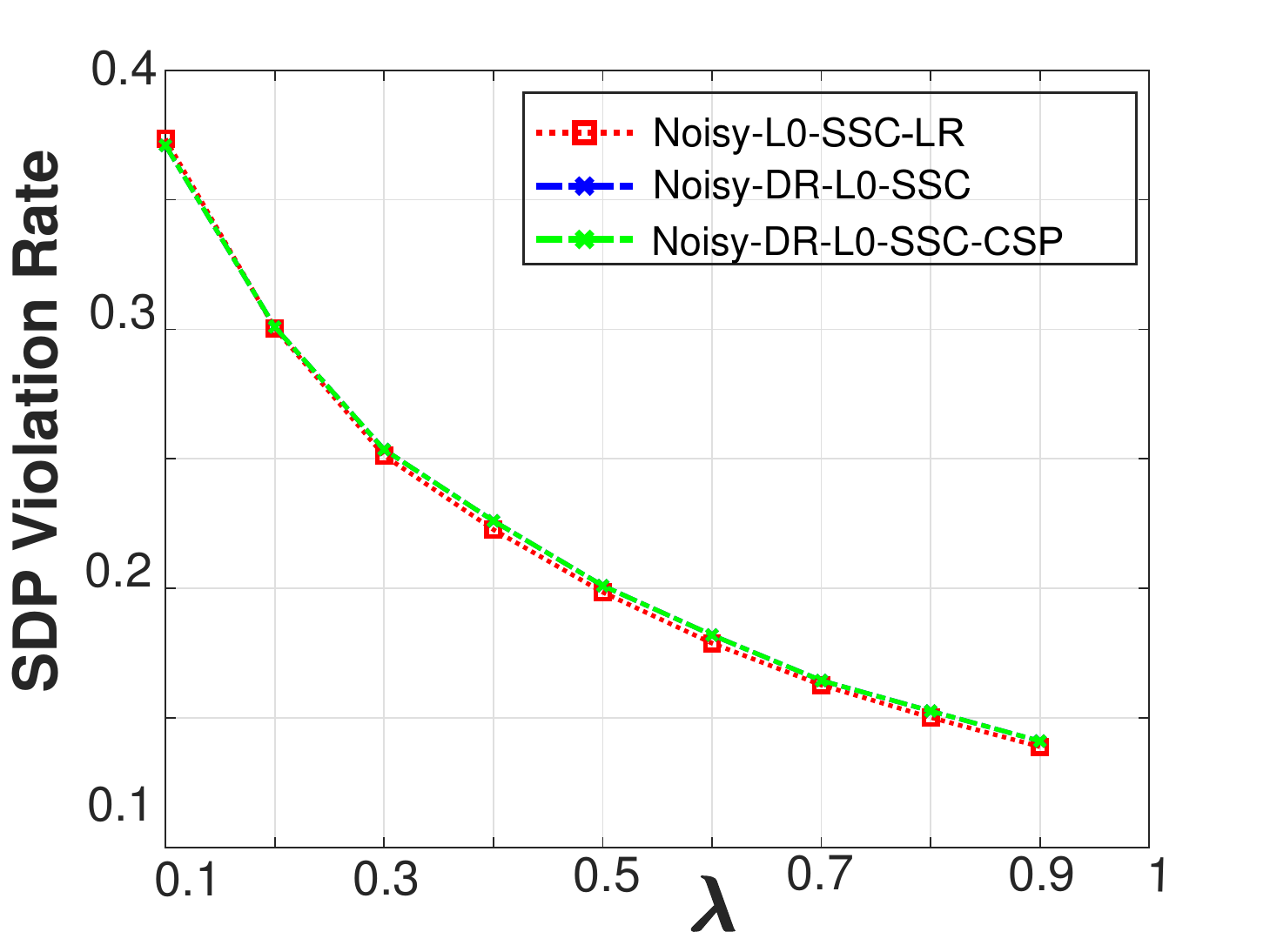}
  \caption{$\sigma^2=60$}
  \label{fig:yaleb-noisy-level-6}
\end{subfigure}
\caption{The SDP violation rate with respect to $\lambda$ for noisy $\ell^0$-SSC, Noisy-DR-$\ell^0$-SSC and Noisy-DR-$\ell^{0}$-SSC-CSP. The SDP violation rate for Noisy-DR-$\ell^0$-SSC and that for Noisy-DR-$\ell^{0}$-SSC-CSP are the same, so their curves overlap each other.}
\label{fig:yaleb-noisy-level}
\end{figure*}

\vspace{-0.06in}
\section{Conclusion}
\vspace{-0.06in}

In this paper, we prove that noisy $\ell^{0}$-SSC recovers subspaces from noisy data through $\ell^{0}$-induced sparsity. Our results for the first time reveal the theoretical advantage of noisy $\ell^{0}$-SSC over its $\ell^1$ counterpart and other competing subspace clustering methods in terms of much less restrictive condition on the subspace affinity, when the size of data grows exponentially in the subspace dimension. We then propose Noisy-DR-$\ell^{0}$-SSC to improve the efficiency of noisy $\ell^{0}$-SSC, which performs noisy $\ell^{0}$-SSC on dimensionality reduced data and still provably recovers the underlying subspaces. Experiments evidence the findings of our theoretical results in the robustness of noisy $\ell^{0}$-SSC against noise as well as the effectiveness of Noisy-DR-$\ell^{0}$-SSC.

\bibliography{refs}

\vspace{1.6in}

\noindent\textbf{\large Appendix}\\

\vspace{0.1in}

\section{Proofs}
We provide proofs to the lemmas and theorems in the paper in this subsection.

\subsection{Lemma~\ref{lemma::l0ssc-deterministic} and Its Proof}
\begin{lemma}\label{lemma::l0ssc-deterministic}
{\rm (Subspace detection property holds for noiseless $\ell^{0}$-SSC under the deterministic model)}
It can be verified that the following statement is true. Under the deterministic model, suppose data is noiseless, $n_k \ge d_k+1$, $\bY^{(k)}$ is in general position. If all the data points in $\bY^{(k)}$ are away from the external subspaces for any $1 \le k \le K$, then the subspace detection property for $\ell^{0}$-SSC holds with an optimal solution $\bZ^*$ to (\ref{eq:l0ssc}).
\end{lemma}
\begin{proof}
Let $\bx_i \in \cS_k$. Note that ${\bZ^*}^{i}$ is an optimal solution to the following $\ell^{0}$ sparse representation problem
\bals
\mathop {\min }\limits_{{\bZ^i}} \norm{\bZ^i}{0}\quad s.t.\;{ \bx_i} = {{[ \bX^{(k)}\setminus  \bx_i \quad  \bX^{(-k)}]}}{\bZ^i},\,\, \bZ_{ii} = 0,
\eals%%
where $\bX^{(-k)}$ denotes the data that lie in all subspaces except $\cS_k$. Let ${\bZ^*}^{i} = \left[ {\begin{array}{*{20}{c}}
\balpha\\
\bbeta
\end{array}} \right]$ where $\balpha$ and $\bbeta$ are sparse codes corresponding to $ \bX^{(k)}\setminus  \bx_i$ and $ \bX^{(-k)}$ respectively.

Suppose $\bbeta \neq \bzero$, then $ \bx_i$ belongs to a subspace $\cS^{'} = \bH_{ \bX_{{\bZ^*}^i}}$ spanned by the projected data points corresponding to nonzero elements of  ${\bZ^*}^{i}$, and $\cS^{'} \neq  \cS_k$, ${\rm dim}[\cS^{'}] \le  d_k$. To see this, if $\cS^{'} = \cS_k$, then the data corresponding to nonzero elements of $\bbeta$ belong to $ \cS_k$, which is contrary to the definition of $\bX^{(-k)}$. Also, if ${\rm dim}[\cS^{'}] >  d_k$, then any $ d_k$ points in $ \bX^{(k)}$ can be used to linearly represent $ \bx_i$ by the condition of general position, contradicting with the optimality of ${\bZ^*}^{i}$.

Since the data points (or columns) in $ \bX_{{\bZ^*}^i}$ are linearly independent, it follows that $\bx_i$ lies in an external subspace $\bH_{\bX_{{\bZ^*}^i}}$ spanned by linearly independent points in $\bX_{{\bZ^*}^i}$, and ${\rm dim}[\bH_{\bX_{{\bZ^*}^i}}] = {\rm dim}[\cS^{'}] \le  d_k$. This contradicts with the assumption that $\bx_i$ is away from the external subspaces. Therefore, $\bbeta = \bzero$. Perform the above analysis for all $1 \le i \le n$, we can prove that the subspace detection property holds for all $1 \le i \le n$.
\end{proof}

\newpage

\subsection{Proof of Theorem~\ref{theorem::noisy-l0ssc-subspace-detection}}
Before proving this theorem, we introduce the following perturbation bound for the distance between a data point and the subspaces spanned by noisy and noiseless data, which is useful to establish the conditions when the subspace detection property holds for noisy $\ell^{0}$-SSC.

\begin{lemma}\label{lemma::perturbation-distance-to-subspace}
Let $\bbeta \in \RR^n$ and $\bY_{\bbeta}$ has full column rank. Suppose $\delta < \bar \sigma_{\bY,r}$ where $r = \norm{\bbeta}{0}$, then $\bX_{\bbeta}$ is a full column rank matrix, and
\bal\label{eq:perturbation-distance-to-subspace}
|d(\bx_i, \bH_{\bX_{\bbeta}}) - d(\bx_i, \bH_{\bY_{\bbeta}}) | \le \frac{\delta} {\bar \sigma_{\bY,r} - \delta}
\eal%%
for any $1 \le i \le n$.
\end{lemma}

Lemma~\ref{lemma::equivalence-noisy-l0ssc} shows that an optimal solution to the noisy $\ell^{0}$-SSC problem (\ref{eq:noisy-l0ssc-i}) is also that to a $\ell^{0}$-minimization problem with tolerance to noise.

\begin{lemma}\label{lemma::equivalence-noisy-l0ssc}
Let nonzero vector $\bbeta^*$ be an optimal solution to the noisy $\ell^{0}$-SSC problem (\ref{eq:noisy-l0ssc-i}) for point $\bx_i$ with $\|\bbeta^*\|_0=r^* > 1$. If $\lambda > \tau_0$ where $\tau_0$ is defined as
\begin{align*}
&\tau_0 \defeq \frac{2\delta\sqrt{r^*}}{\sigma_{\bX}^*} + \tau_1,
\end{align*}%
where
\begin{align*}%\label{eq:equivalence-noisy-l0ssc-tau1}
&\tau_1 \defeq \frac{\delta} {\bar \sigma_{\bY}^* - \delta}, \quad \sigma_{\bX}^* \defeq \sigma_{\min}(\bX_{\bbeta^*}),
\end{align*}%
with $\delta < \bar \sigma_{\bY}^*$, and $\bar \sigma_{\bY}^*$ is defined as
\begin{align*}
&\bar \sigma_{\bY}^* \defeq \min_{r \in [r^*]} \bar \sigma_{\bY,r},
\end{align*}%
then $\bbeta^*$ is an optimal solution to the following sparse approximation problem with the uncorrupted data as the dictionary:
\bal\label{eq:equivalence-noisy-l0ssc-2}
&\mathop {\min }\limits_{{\bbeta}} {\norm{{\bbeta}}{0}} \quad s.t.\;\ltwonorm{{\bx_i} - {{\bY}}{\bbeta}} \le c^*+\frac{2\delta\sqrt{r^*}}{\sigma_{\bX}^*},  \,\, \bbeta_i = 0,
\eal%%
where $c^* \defeq \|\bx_i - \bX \bbeta^*\|_2$.
\end{lemma}

Now we are ready to prove Theorem~\ref{theorem::noisy-l0ssc-subspace-detection}.

\begin{proof}[\textup{\bf Proof of Theorem~\ref{theorem::noisy-l0ssc-subspace-detection}}]
We first show that $d(\bx_i, \cS_k) \le c^*+\frac{2\delta\sqrt{r^*}}{\sigma_{\bX}^*}$. To see this, $\sigma_{\bX}^* = \sigma_{\min}(\bX_{\bbeta^*}) \le 1$ as the columns of $\bX$ have unit $\ell^{2}$-norm. It follows that
\bals
&c^*+\frac{2\delta\sqrt{r^*}}{\sigma_{\bX}^*} \ge 2 \delta \sqrt{r^*} \ge 2 \delta > \|\bx_i - \by_i\| \ge d(\bx_i, \cS_k).
\eals%%

By Lemma~\ref{lemma::equivalence-noisy-l0ssc}, it can be verified that $\bbeta^*$ is an optimal solution to the following problem
\bal\label{eq:noisy-l0ssc-subspace-detection-seg2}
&\mathop {\min }\limits_{{\bbeta}} {\norm{\bbeta}{0}} \quad s.t.\;\ltwonorm{{\bx_i} - {{\bY}}{\bbeta}} \le c^*+\frac{2\delta\sqrt{r^*}}{\sigma_{\bX}^*},  \,\, \bbeta_i = 0.
\eal%
Let $\bx'$ be the projection of $\bx_i$ onto $\bH_{\overbar {\bY^{(i)}}}$, and let the columns of $\overbar {\bY^{(i)}}$ have column indices $\bI$ in $\bY^{(k)}$, that is, $\bY_{\bI}^{(k)} = \overbar {\bY^{(i)}}$. Then there exists $\bbeta' \in \RR^n$ and $\bbeta'_j = 0$ for all $j \notin \bI$ such that $\bx' = \bY \bbeta'$ and $\norm{\bbeta'}{0} \le r^*$. It is clear that $\bbeta'$ is a feasible solution to (\ref{eq:noisy-l0ssc-subspace-detection-seg2})  because $d(\bx_i,\bH_{\overbar {\bY^{(i)}}}) = \ltwonorm{\bx_i - \bY \bbeta'} \le c^*+\frac{2\delta\sqrt{r^*}}{\sigma_{\bX}^*}$ and it satisfies SDP for $\bx_i$.

Suppose that there is an optimal solution $\bbeta''$ to (\ref{eq:noisy-l0ssc-subspace-detection-seg2}) which does not satisfy SDP for $\bx_i$, then $\norm{\bbeta''}{0} \le r^*$. Then the subspace spanned by $\bY_{\bbeta''}$, $\bH_{\bY_{\bbeta''}}$, is an external subspace of $\by_i$ and $\bH_{\bY_{\bbeta''}} \in \cH_{\by_i,r^*}$, and it follows that $d(\bx_i,\bH_{\bY_{\bbeta''}}) > c^*+\frac{2\delta\sqrt{r^*}}{\sigma_{\bX}^*}$. However, since $\bbeta''$ is a feasible solution, $d(\bx_i,\bH_{\bY_{\bbeta''}}) \le c^*+\frac{2\delta\sqrt{r^*}}{\sigma_{\bX}^*}$. This contradiction shows that every optimal solution to the noisy $\ell^{0}$-SSC problem (\ref{eq:noisy-l0ssc-i}) satisfies SDP for $\bx_i$.

\end{proof}

\subsection{Proof of Lemma~\ref{lemma::perturbation-distance-to-subspace}}

The following lemma is used for proving Lemma~\ref{lemma::perturbation-distance-to-subspace}.

\begin{lemma}\label{lemma::perturbation-distance-to-hyperplane}
{\rm (Perturbation of distance to subspaces)}
Let $\bA$, $\bB \in \RR^{m \times n}$ are two matrices and ${\rm rank}(\bA) = r$, ${\rm rank}(\bB) = s$. Also, $\bE = \bA-\bB$ and $\|\bE\|_2 \le C$, where $\|\cdot\|_2$ indicates the spectral norm. Then for any point $\bx \in \RR^m$, the difference of the distance of $\bx$ to the column space of $\bA$ and $\bB$, i.e. $|d(\bx, \bH_{\bA}) - d(\bx, \bH_{\bB})|$, is bounded by
\bals
\abth{d(\bx, \bH_{\bA}) - d(\bx, \bH_{\bB}) } \le \frac{C \ltwonorm{\bx} }{ \min\{\sigma_{r}(\bA),\sigma_{s}(\bB)\} }.
\eals%%
\end{lemma}
\begin{proof}
Note that the projection of $\bx$ onto the subspace $\bH_{\bA}$ is $\bA \bA^{+} \bx$ where $\bA^{+}$ is the Moore-Penrose pseudo-inverse of the matrix $\bA$, so $d(\bx, \bH_{\bA})$ equals to the distance between $\bx$ and its projection, namely $d(\bx, \bH_{\bA}) = \ltwonorm{\bx - \bA \bA^{+} \bx}$. Similarly,
$d(\bx, \bH_{\bB}) = \ltwonorm{\bx - \bB \bB^{+} \bx}$.

It follows that
\bal\label{eq:perturbation-distance-to-hyperplane-seg1}
&\abth{d(\bx, \bH_{\bA}) - d(\bx, \bH_{\bB}) }= \abth{ \ltwonorm{\bx - \bA \bA^{+} \bx} -  \ltwonorm{\bx - \bB \bB^{+} \bx} } \nonumber \\
&\le \ltwonorm{\bA \bA^{+} \bx - \bB \bB^{+} \bx} \le \ltwonorm{\bA \bA^{+} - \bB \bB^{+}} \ltwonorm{\bx}.
\eal%%

According to the perturbation bound on the orthogonal projection in~\citet{Chen2016-perturbation-orthogonal-projection,Stewart1977-perturbation-pseudoinverse-projection},
\bal\label{eq:perturbation-distance-to-hyperplane-seg2}
&\|\bA \bA^{+} - \bB \bB^{+}\|_2 \le \max\{\|\bE\bA^{+}\|_2, \|\bE\bB^{+}\|_2\}.
\eal%%

Since $\ltwonorm{\bE\bA^{+}} \le \ltwonorm{\bE} \ltwonorm{\bA^{+}} \le \frac{C}{\sigma_{r}(\bA)}$,
$\ltwonorm{\bE\bB^{+}} \le \ltwonorm{\bE} \ltwonorm{\bB^{+}} \le \frac{C}{\sigma_{s}(\bB)}$, combining (\ref{eq:perturbation-distance-to-hyperplane-seg1}) and (\ref{eq:perturbation-distance-to-hyperplane-seg2}), we have

\bals
\abth{d(\bx, \bH_{\bA}) - d(\bx, \bH_{\bB})} \le \max\{ \frac{C}{\sigma_{r}(\bA)},\frac{C}{\sigma_{s}(\bB)}\} \|\bx\|_2 \nonumber  =\frac{C \|\bx\|_2}{ \min\{\sigma_{r}(\bA),\sigma_{s}(\bB)\} }.
\eals%%
So that (\ref{eq:perturbation-distance-to-subspace}) is proved.
\end{proof}

\newpage

\begin{proof}[\textup{\bf {Proof of Lemma~\ref{lemma::perturbation-distance-to-subspace}}}]
We have $\by_i = \bx_i - \bn_i$, and $\sigma_{\min}(\bY_{\bbeta}^{\top}\bY_{\bbeta})  = \big(\sigma_{\min}(\bY_{\bbeta}) \big)^2 \ge \sigma_{\bY,r}^2$.

By Weyl~\citep{Weyl1912-perturbation-singular-value}, $|\sigma_{i}(\bY_{\bbeta}) - \sigma_{i}(\bX_{\bbeta})| \le \ltwonorm{\bN_{\bbeta}} \le \norm{\bN_{\bbeta}}{F} \le \sqrt{r} \delta$. Since $ \sqrt{r} \delta < \sigma_{\bY,r} \le \sigma_{\min}(\bY_{\bbeta}) \le \sigma_{i}(\bY_{\bbeta})$, $\sigma_{i}(\bX_{\bbeta}) \ge \sigma_{i}(\bY_{\bbeta}) - \sqrt{r} \delta \ge \sigma_{\bY,r} - \sqrt{r} \delta > 0$ for $1 \le i \le \min\{d,r\}$. It follows that $\sigma_{\min}(\bX_{\bbeta}) \ge  \sigma_{\bY,r} - \sqrt{r} \delta > 0$ and $\bX_{\bbeta}$ has full column rank.

Also, $\|\bX_{\bbeta} - \bY_{\bbeta}\|_2 \le \|\bX_{\bbeta} - \bY_{\bbeta}\|_F \le \sqrt{r} \delta$. According to Lemma~\ref{lemma::perturbation-distance-to-hyperplane},
\bals
\abth{d(\bx_i, \bH_{\bX_{\bbeta}}) - d(\bx_i, \bH_{\bY_{\bbeta}}) }
&\le \frac{\sqrt{r} \delta}{ \min\{\sigma_{\min}(\bX_{\bbeta}),\sigma_{\min}(\bY_{\bbeta})\} } \nonumber \\
&\le \frac{\sqrt{r} \delta} {\sigma_{\bY,r} - \sqrt{r} \delta} = \frac{\delta} {\bar \sigma_{\bY,r} - \delta}.
\eals%

\end{proof}

\subsection{Proof of Lemma~\ref{lemma::equivalence-noisy-l0ssc}}
\begin{proof}[\textup{\bf Proof of Lemma~\ref{lemma::equivalence-noisy-l0ssc}}]
We have
\begin{align*}
&\ltwonorm{\bx_i - \bX \bbeta^*}^2 + \lambda \norm{\bbeta^*}{0} \le \ltwonorm{\bx_i - \bX \bzero}^2 + {\lambda}\norm{\bzero}{0} = 1 \\
&\Rightarrow c^* = \ltwonorm{\bx_i - \bX \bbeta^*} \le \sqrt{1 - \lambda r^*} < 1.
\end{align*}

We first prove that $\bbeta^*$ is an optimal solution to the sparse approximation problem
\bal\label{eq:equivalence-noisy-l0ssc-1}
&\mathop {\min }\limits_{{\bbeta}} {\norm{\bbeta}{0}} \quad s.t.\;\ltwonorm{{\bx_i} - {{\bX}}{\bbeta}} \le c^*,  \,\, \bbeta_i = 0.
\eal%%

To see this, if $r^* = 1$, then $\beta^*$ must be an optimal solution to (\ref{eq:equivalence-noisy-l0ssc-1}). If $r^* > 1$, suppose there is a vector $\bbeta'$ such that $\ltwonorm{{\bx_i} - {{\bX}}{\bbeta'}} \le c^*$ and $\norm{\bbeta'}{0} < \norm{\bbeta^*}{0}$, then $L(\bbeta') < c^* + \lambda \norm{\bbeta^*}{0} = L(\bbeta^*)$, contradicting the fact that $\bbeta^*$ is an optimal solution to (\ref{eq:noisy-l0ssc-i}).

Note that $\bX_{\bbeta^*}$ is a full column rank matrix, otherwise a sparser solution to (\ref{eq:noisy-l0ssc-i}) can be obtained as vector whose support corresponds to the maximal linear independent set of columns of $\bX_{\bbeta^*}$.

Also, the distance between $\bx_i$ and the subspace spanned by columns of $\bX_{\bbeta^*}$ equals to $c^*$, i.e. $d(\bx_i,\bH_{\bX_{\bbeta^*}}) = c^*$. To see this, it is clear that $d(\bx_i,\bH_{\bX_{\bbeta^*}}) \le c^*$. If there is a vector $\by = \bX {\tilde \bbeta}$ in $\bH_{\bX_{\bbeta^*}}$ with ${\rm supp}({\tilde \bbeta}) \subseteq {\rm supp}({\bbeta^*})$, and $\ltwonorm{\bx_i-\by} < c^*$, then $L({\tilde \bbeta}) < L(\bbeta^*)$ which contradicts the optimality of $\bbeta^*$. Therefore, $d(\bx_i,\bH_{\bX_{\bbeta^*}}) \ge c^*$, and it follows that $d(\bx_i,\bH_{\bX_{\bbeta^*}}) = c^*$.

Since $\ltwonorm{\bx_i - \bX \bbeta^*} \le 1$, $\ltwonorm{\bX \bbeta^*} \le 2$. Also,
$$\sigma_{\min} (\bX_{\bbeta^*}^{\top}{\bX_{\bbeta^*}}) \ltwonorm{\bbeta^*}^2 \le \|\bX \bbeta^*\|_2^2 \le 4,$$
it follows that $\ltwonorm{\bbeta^*}^2 \le \frac{4}{{\sigma_{\bX}^*}^2}$. By Cauchy-Schwarz inequality, $\lonenorm{\bbeta^*} \le \frac{2\sqrt{r^*}}{\sigma_{\bX}^*}$ and $\ltwonorm{\bN \bbeta^*} \le \lonenorm{\bbeta^*} \delta \le \frac{2\delta\sqrt{r^*}}{\sigma_{\bX}^*}$. Therefore,

\bals
\ltwonorm{{\bx_i} - {\bY}{\bbeta^*}} &= \ltwonorm{{\bx_i} - {{\bX}}{\bbeta^*} + {{\bN}}{\bbeta^*}}  \\
& \le \ltwonorm{{\bx_i} - {{\bX}}{\bbeta^*}} + \ltwonorm{{{\bN}}{\bbeta^*}} \le c^* + \frac{2\delta\sqrt{r^*}}{\sigma_{\bX}^*},
\eals%
so that $\bbeta^*$ is a feasible for problem (\ref{eq:equivalence-noisy-l0ssc-2}).

To prove that $\bbeta^*$ is an optimal solution to (\ref{eq:equivalence-noisy-l0ssc-2}), we first note that $\bbeta^*$ must be an optimal solution to (\ref{eq:equivalence-noisy-l0ssc-2}) if $r^*=1$. This is because $c^* \le \sqrt{1 - \lambda r^*} \le 1 - \lambda$ and $\lambda > \tau_0 > \frac{2\delta\sqrt{r^*}}{\sigma_{\bX}^*}$ so that $c^* + \frac{2\delta\sqrt{r^*}}{\sigma_{\bX}^*} < 1$, and it follows that $\bzero$ is not feasible to (\ref{eq:equivalence-noisy-l0ssc-2}).

If $r^* > 1$ and suppose $\bbeta^*$ is not an optimal solution to (\ref{eq:equivalence-noisy-l0ssc-2}), then an optimal solution to (\ref{eq:equivalence-noisy-l0ssc-2}) is a vector $\bbeta'$ such that $\|{\bx_i} - {{\bY}}{\bbeta'}\|_2 \le c^*+ \frac{2\delta\sqrt{r^*}}{\sigma_{\bX}^*}$ and $\| {\bbeta'} \|_0 = r < r^*$.

${\bY}_{\bbeta'}$ is a full column rank matrix, otherwise a sparser solution can be obtained as vector whose support corresponds to the maximal linear independent set of columns of $\bY_{\bbeta'}$. We have
\begin{align*}
&d(\bx_i, \bH_{\bY_{\bbeta'}}) \le \ltwonorm{{\bx_i} - {\bY}{\bbeta'}} \le c^* + \frac{2\delta\sqrt{r^*}}{\sigma_{\bX}^*}.
\end{align*}%
According to Lemma~\ref{lemma::perturbation-distance-to-subspace}, we have

\bal\label{eq:equivalence-noisy-l0ssc-seg1}
&\abth{d(\bx_i, \bH_{\bX_{\bbeta'}}) - d(\bx_i, \bH_{\bY_{\bbeta'}})} \le \frac{\sqrt{r} \delta} {\sigma_{\bY,r} - \sqrt{r} \delta}
 = \frac{\delta} {\bar \sigma_{\bY,r} - \delta} \le \frac{\delta} {\bar \sigma_{\bY}^* - \delta} \\
&\Rightarrow d(\bx_i, \bH_{\bX_{\bbeta'}}) \le c^* + \frac{2\delta\sqrt{r^*}}{\sigma_{\bX}^*} + \frac{\delta} {\bar \sigma_{\bY}^* - \delta} = c^* + \tau_0.
\eal%
However, according to the optimality of $\bbeta^*$ in the noisy $\ell^{0}$-SSC problem (\ref{eq:noisy-l0ssc-i}), we have

\bal\label{eq:equivalence-noisy-l0ssc-seg2}
d(\bx_i, \bH_{\bX_{\bbeta'}}) - c^* &= d(\bx_i, \bH_{\bX_{\bbeta'}}) - d(\bx_i, \bH_{\bX_{\bbeta^*}}) \nonumber \\
&\stackrel{\circled{1}}{\ge} (r^* - r) \lambda > \tau_0.
\eal%
To see $\circled{1}$ holds, let $\bbeta'' \in \RR^d$, $\supp{\bbeta''} \subseteq \supp{\bbeta'}$ such that $\ltwonorm{\bx_i - \bX \bbeta''} = d(\bx_i, \bH_{\bX_{\bbeta'}})$.
Then by the optimality of $\bbeta^*$,
\bals
\ltwonorm{\bx_i - \bX \bbeta''}  &\ge  d(\bx_i, \bH_{\bX_{\bbeta^*}}) + \lambda r^* - \lambda \abth{\supp{\bbeta''}} \\
&\ge d(\bx_i, \bH_{\bX_{\bbeta^*}}) + (r^* - r) \lambda.
\eals
The contradiction between (\ref{eq:equivalence-noisy-l0ssc-seg1}) and (\ref{eq:equivalence-noisy-l0ssc-seg2}) shows that $\bbeta^*$ is an optimal solution to (\ref{eq:equivalence-noisy-l0ssc-2}).
\end{proof}

\subsection{Proof of Theorem~\ref{theorem::noisy-l0ssc-subspace-detection-lambda}}
\begin{proof}[\textup{Proof of Theorem~\ref{theorem::noisy-l0ssc-subspace-detection-lambda}}]
This theorem can be proved by checking that the conditions in Theorem~\ref{theorem::noisy-l0ssc-subspace-detection} are satisfied.
\end{proof}

\subsection{Proof of Theorem~\ref{theorem::noisy-l0ssc-subspace-detection-lambda-random}}

In order to prove this theorem, the following lemma is presented and it provides the geometric concentration inequality for the distance between a point $\by \in \bY^{(k)}$ and any of its external subspaces. It renders a lower bound for $M_{i}$, namely the minimum distance between $\by_i \in \cS_k$ and its external subspaces.

\begin{lemma}\label{lemma::point-to-subspace-concentration}
Under semi-random model, given $1 \le k \le K$ and $\by \in \bY^{(k)}$, suppose $\bH \in \cH_{\by_i, d_k}$ is any external subspace of $\by$. Moreover, assume that for any external subspace $\bH'$ of $\by$, ${\rm Tr} ( \bU^{\top}_{\bH} \bU^{(k)} {\bU^{(k)}}^{\top} \bU_{\bH} ) \le d_k - 1$ where $\bU_{\bH}$ is an orthonormal basis of $\bH$.  Then for any $t>0$,
\bal\label{eq:point-to-subspace-concentration}
&\Pr[d(\by, \bH) \ge \frac 1 {d_k} - 2t\sqrt{1 - \frac 1 {d_k}}-t^2] \ge 1-8\exp(-\frac{d_k t^2}{2}).
\eal%%
\end{lemma}
\begin{proof}[\textup{\bf Proof of Lemma~\ref{lemma::point-to-subspace-concentration}}]
Let $\bH$ be a fixed subspace of dimension $d_e \le d_k$, and $\by \notin \bH$. Since $\by \in \cS_k$ and $\by \notin \bH$. Let $\by = \bU^{(k)} \tilde \by$ and $\Expect{}{\tilde \by {\tilde \by}^{\top}} = \bI_{d_k}$.

Then the projection of $\by$ onto $\bH$ is $\mathbb P_{\bH}(\by) = \bU_{\bH} \bU^{\top}_{\bH} \by$, and we have
\begin{small}\bal\label{eq:point-to-subspace-concentration-seg1}
&\E[\ltwonorm{\mathbb P_{\bH}(\by)}^2] = \E[ \by^{\top} \bU_{\bH} \bU^{\top}_{\bH} \bU_{\bH} \bU^{\top}_{\bH} \by] \nonumber \\
&= \E[{\rm Tr}(\by^{\top} \bU_{\bH} \bU^{\top}_{\bH} \by )] \nonumber \\
&= \E[{\rm Tr}( \bU^{\top}_{\bH} \by \by^{\top} \bU_{\bH}  )] \nonumber \\
&={\rm Tr} ( \bU^{\top}_{\bH} \E[\by \by^{\top} ] \bU_{\bH} ) \nonumber \\
&={\rm Tr} ( \bU^{\top}_{\bH} \bU^{(k)} \E[\tilde \by {\tilde \by}^{\top} ] {\bU^{(k)}}^{\top} \bU_{\bH} ) \nonumber \\
&=\frac{1}{d_k} {\rm Tr} ( \bU^{\top}_{\bH} \bU^{(k)} {\bU^{(k)}}^{\top} \bU_{\bH} ) \le \frac{d_k-1}{d_k} = 1 - \frac 1{d_k}.
\eal\end{small}%%
According to the concentration inequality in section 5.2 of~\citep{aubrun2017alice}, for any $t > 0$,
\bal\label{eq:point-to-subspace-concentration-seg2}
&\Pr[\abth{\ltwonorm{\mathbb P_{\bH}(\by)} - \Expect{}{\ltwonorm{\mathbb P_{\bH}(\by)}}} \ge t] \le 8 \exp(-\frac{d_k t^2}{2}),
\eal%%

and by (\ref{eq:point-to-subspace-concentration-seg1}) $\Expect{}{\ltwonorm{\mathbb P_{\bH}(\by)}} \le \sqrt{1 - \frac 1{d_k}}$.

Now let $\bH$ be spanned by data from $\bY$, i.e. $\bH = \bH_{\{\by_{i_j}\}_{j=1}^{d_e}}$, where $\{\by_{i_j}\}_{j=1}^{d_e}$ are any $d_e$ linearly independent points that does not contain $\by$. For any fixed points $\set{\by_{i_j}}_{j=1}^{d_e}$, (\ref{eq:point-to-subspace-concentration-seg2}) holds. Let $A$ be the event that $\abth{\mathbb P_{\bH}(\by) - \Expect{}{\ltwonorm{\mathbb P_{\bH}(\by)}}} \ge t$, we aim to integrate the indicator function $\1_{A}$ with respect to the random vectors, i.e. $\by$ and $\{\by_{i_j}\}_{j=1}^{d_e}$, to obtain the probability that $A$ happens over these random vectors. Let $\by = \by_i$, using Fubini theorem, we have
\bal\label{eq:point-to-subspace-concentration-seg3}
\Pr[A] &= \int_{\otimes_{j=1}^n \cS^{(j)}} \indict{A} {\otimes}_{j=1}^n d\mu^{(j)} \nonumber \\
&=\int_{\otimes_{j \neq i} \cS^{(j)}} \Pr[A | \{\by_j\}_{j \neq i}] {\otimes}_{j \neq i} d\mu^{(j)} \nonumber \\
&\le \int_{\otimes_{j \neq i} \cS^{(j)}}  8 \exp(-\frac{d_k t^2}{2}) {\otimes}_{j \neq i} d\mu^{(j)} = 8 \exp(-\frac{d_k t^2}{2}),
\eal%%
where $\cS^{(j)} \in \{\cS_k\}_{k=1}^K$ is the subspace that $\by_j$ lies in, and $\mu^{(j)}$ is the probabilistic measure of the distribution in $\cS^{(j)}$. The last inequality is due to (\ref{eq:point-to-subspace-concentration-seg2}).

Note that for any $\by$'s external subspace $\bH = \bH_{\{\by_{i_j}\}_{j=1}^{d_e}}$, $d(\by, \bH) = \sqrt{\|\by\|_2^2 - \|\mathbb P_{\bH}(\by)\|_2^2} = \sqrt{1 - \|\mathbb P_{\bH}(\by)\|_2^2} $. According to (\ref{eq:point-to-subspace-concentration-seg3}), we have
\bals
&\Pr[d(\by, \bH) \ge \frac 1 {d_k} - 2t\sqrt{1 - \frac 1 {d_k}}-t^2] \ge 1-8\exp(-\frac{d_k t^2}{2}).
\eals%%
\end{proof}

\newpage

The following lemma shows the lower bound for any submatrix of $\bY^{(k)}$.

\begin{lemma}\label{lemma:chi-square-concentration}
{\rm (\cite[Lemma 1]{Laurent2000-chi-square-concentration})} Let $\set{X_i}_{i=1}^k$ be i.i.d. standard Gaussian random variables and $X = \sum\limits_{i=1}^k X_i^2$, then
\bals%\label{eq:chi-square-concentration}
\Prob{X - k \ge 2 \sqrt{kx} + 2x} &\ge \exp\pth{-x}, \nonumber \\
\Prob{k - X \ge 2 \sqrt{kx} } &\ge \exp\pth{-x}.
\eals%
\end{lemma}

\begin{lemma}\label{lemma::spectrum-bound-random-matrix}
{\rm (Spectrum bound for Gaussian random matrix, \cite[Theorem II.13]{Davidson08-local-operator-random-matrix})}
Suppose $\bA \in \RR^{m \times n}$ ($m \ge n$) is a random matrix whose entries are i.i.d. samples generated from the standard Gaussian distribution $\cN(0,\frac{1}{m})$. Then
\bsals%\label{eq:spectrum-bound-1}
& 1 - \sqrt {\frac{n}{m}} \le \E[\sigma_{n}(\bA)] \le \E[\sigma_{1}(\bA)] \le 1 + \sqrt {\frac{n}{m}}.
\esals%
Also, for any $t > 0$,
\bal
&\Pr[\sigma_{n}(\bA) \le 1 - \sqrt {\frac{n}{m}} - t] < \exp\pth{-\frac{mt^2}{2}}, \label{eq:least-singular-lower-bound} \\
&\Pr[\sigma_{1}(\bA) \ge 1 + \sqrt {\frac{n}{m}} + t] < \exp\pth{-\frac{mt^2}{2}}. \nonumber
\eal%
\end{lemma}

\begin{lemma}\label{lemma:lower-bound-singular-clean-data}
Let $\bY \in \RR^{d \times r}$ be any submatrix of $\bY^{(k)}$ with ${\rm rank}(\bY) = r$ and $r \le r_0 \le \floor{\frac{1}{\lambda}} \le d_k$, $k \in [K]$. Suppose $c_1 > 0$ is an arbitrary small constant, $\eps_0, \eps_1 > 0$ be small constants, and $d_k$ is large enough such that $2d^{-0.05}_k + 2d^{-0.1}_k \le \eps_0$ and $\sqrt{\frac{1}{\lambda d_k}} + \sqrt{ \frac{2}{\lambda d_k} \log{\frac{en_k}{r_0}}} \le \eps_1$.  Then with probability at least $1-\exp(-c_1 d_k)-2n_k\exp\pth{-d^{0.9}_k}$, $\sigma_{\min}(\bY) \ge \sigma'_{\min} $, where $\sigma'_{\min}$ is defined by (\ref{eq:sigma-min-clean-data}).
\end{lemma}
\begin{proof}
Let $\bY = \bU^{(k)} \balpha \bS$ be a submatrix of size $d_k \times r$ of $\bY^{(k)}$. $\balpha \in \RR^{d_k \times r}$ and elements of $\balpha$ are i.i.d. standard Gaussians, that is, $\balpha_{ij} \sim \cN(0,1)$, $i \in [d_k], j \in [r]$. $\bS \in \RR^{r \times r}$ is a diagonal matrix with $\bS_{ii} = \ltwonorm{\balpha^i}$ for $i \in [r]$. Define $\bC \defeq \balpha \bS$. By the concentration property of $\chi^2$-distribution (Lemma~\ref{lemma:chi-square-concentration}), with probability at least $1-2n_k\exp\pth{-d^{0.9}_k}$, $\bS_{ii} \in [\sqrt{d_k - 2d^{0.95}_k}, \sqrt{d_k + 2d^{0.95}_k + 2d^{0.9}_k}]$ for all $i \in [r]$ and any submatrix $\bY$ of $\bY^{(k)}$.

Now we estimate an lower bound for the least singular value of $\balpha$. By (\ref{eq:least-singular-lower-bound}) of Lemma~\ref{lemma::spectrum-bound-random-matrix}, for a particular submatrix $\bY$ of $\bY^{(k)}$ and the corresponding $\balpha$ and any $t > 0$, we have
\bsal\label{eq:lower-bound-singular-clean-data-seg1}
\Prob{\sigma_{\min}(\balpha) \ge \sqrt{d_k} - \sqrt{r} - \sqrt{d_k} t } \ge 1 - \exp\pth{-\frac{d_k t^2}{2}}. \esal%
Now there are $\binom{n_k}{r}$ ways of choosing the submatrix $Y$, and $\binom{n_k}{r} \le \pth{\frac{en_k}{r}}^r$. Applying the union bound to (\ref{eq:lower-bound-singular-clean-data-seg1}), we have
\bsal\label{eq:lower-bound-singular-clean-data-seg2}
\Prob{\sigma_{\min}(\balpha) \ge \sqrt{d_k} - \sqrt{r} - \sqrt{d_k} t } &\ge 1 - \binom{n_k}{r} \exp\pth{-\frac{d_k t^2}{2}} \nonumber \\
&\ge 1 - \exp\pth{r \log{\frac{en_k}{r}} - \frac{d_k t^2}{2} } \nonumber \\
&\ge 1 - \exp\pth{r_0 \log{\frac{en_k}{r_0}} - \frac{d_k t^2}{2} }
\esal
for any submatrix $Y \in \RR^{d_k \times r}$ of $\bY^{(k)}$.
Let $c_1 > 0$ and $ t = \frac{\sqrt{2 r_0 \log{\frac{en_k}{r_0}}}}{\sqrt{d_k}}  +\sqrt{c_1}$ in (\ref{eq:lower-bound-singular-clean-data-seg2}), then with probability at least $1 - \exp\pth{-\frac{c_1 d_k}{2}}$, $\sigma_{\min}(\balpha) \ge \sqrt{d_k}(1-\sqrt{c_1}) - \sqrt{r} - \sqrt{2 r_0 \log{\frac{en_k}{r_0}}} $. Combined with the bounds for $\bS_{ii}$, we conclude that with probability at least $1-\exp(-c_1d_k)-2n_k\exp\pth{-d^{0.9}_k}$,
\bsals
\sigma_{\min}(\bY) =  \sigma_{\min}(\balpha \bS) &\ge \frac{\sqrt{d_k}(1-\sqrt{c_1}) - \sqrt{r} - \sqrt{2 r_0 \log{\frac{en_k}{r_0}}}}{\sqrt{d_k + 2d^{0.95}_k + 2d^{0.9}_k}} \\
&\ge \frac{1}{1 + 2d^{-0.05}_k + 2d^{-0.1}_k} \pth{1-\sqrt{c_1} - \sqrt{\frac{r}{d_k}} - \sqrt{ \frac{2r_0}{d_k} \log{\frac{en_k}{r_0}}} } \\
&\ge \frac{1}{1+\eps_0} \pth{1-\sqrt{c_1} - \eps_1} = \sigma'_{\min}.
\esals%
\end{proof}

\begin{proof}[\textup{\bf Proof of Theorem~\ref{theorem::noisy-l0ssc-subspace-detection-lambda-random}}]

Let $\bY_{\bbeta}$ for any $\bbeta \in \RR^n$ with $\|\bbeta\|_0 = r_0$. Noting that $\bY_{\bbeta}$ have columns from at most $r_0$ subspaces, let $\bbeta = \sum_{r=1}^{r_0} \bbeta^{(r)}$, $\set{\bbeta^{(r)}}_{r=1}^{r_0}$ have non-overlapping support, each $\bY_{\bbeta^{(r)}}$ is a submatrix of $\bY_{\bbeta}$ and columns of $\bY_{\bbeta^{(r)}}$ are from the same subspace. For any $\bu \in \RR^{r_0}$ with $\ltwonorm{\bu}=1$, we can write $\bu$ as $\bu = \sum_{r=1}^{r_0} \bu^{(r)}$ where $\set{\bu^{(r)}}_{r=1}^{r_0}$ have non-overlapping support and $\bu^{(r)}$ corresponds to $\bY_{\bbeta^{(r)}}$ for $r \in [r_0]$.
With $d_{\min}$ sufficiently large as specified in the conditions of this theorem, by Lemma~\ref{lemma:lower-bound-singular-clean-data}, $\sigma_{\min}(\bY_{\bbeta^{(r)}}) \ge \sigma'_{\min}$ for $r \in [r_0]$, where $\sigma'_{\min}$ is defined by (\ref{eq:sigma-min-clean-data}). Furthermore, define
\bsals
{\rm aff}_{\max} \defeq \max_{t_1,t_2 \in [K] \colon t_1 \neq t_2} \aff{\cS_{t_1},\cS_{t_2}}.
\esals%

We then have
\bal\label{eq:noisy-l0ssc-subspace-detection-lambda-random-seg1}
\ltwonorm{\bY_{\bbeta} \bu}^2 &= \sum_{r=1}^{r_0} \ltwonorm{ \bY_{\bbeta^{(r)}} \bu^{(r)}}^2 + 2 \sum_{s,t \in [r_0] \colon s < t}  {\bu^{(s)}}^{\top} \bY_{\bbeta^{(s)}}^{\top} \bY_{\bbeta^{(t)}} {\bu^{(t)}}
\nonumber \\
&\ge \sigma'^2_{\min} \ltwonorm{\bu}^2 -  2 \sum_{s,t \in [r_0] \colon s < t} \ltwonorm{\bu^{(s)}} \ltwonorm{\bu^{(t)}} {\rm aff}_{\max} \nonumber \\
&\ge \pth{\sigma'^2_{\min} - (r_0-1){\rm aff}_{\max} } \ltwonorm{\bu}^2 \nonumber \\
&= \sigma'^2_{\min} - (r_0-1){\rm aff}_{\max}.
\eal%
It follows that $\sigma_{\min}(\bY_{\bbeta}) \ge \sigma'^2_{\min} - (r_0-1){\rm aff}(\cS_{t_1}, \cS_{t_2})$.
By Weyl~\citep{Weyl1912-perturbation-singular-value}, $\abth{\sigma_{\min}(\bX_{\bbeta}) - \sigma_{\min}(\bY_{\bbeta})}\le \|\bN_{\bbeta}\|_2 \le \delta \sqrt{r_0}$. Therefore, it follows by (\ref{eq:noisy-l0ssc-subspace-detection-lambda-random-seg1}) that
\bals
&\sigma_{\min}(\bX_{\bbeta}) \ge \sigma'^2_{\min} - (r_0-1){\rm aff}(\cS_{t_1}, \cS_{t_2}) - \delta \sqrt{r_0} > 0,
\eals%
if $\delta < \frac{\sigma'^2_{\min} - (r_0-1){\rm aff}(\cS_{t_1}, \cS_{t_2})}{\sqrt{r_0}} = c$. It can be verified that (\ref{eq:noisy-l0ssc-sdp-random-cond2}), (\ref{eq:noisy-l0ssc-sdp-random-cond3}) and (\ref{eq:noisy-l0ssc-sdp-lambda-random}) guarantee (\ref{eq:noisy-l0ssc-sdp-M}), (\ref{eq:noisy-l0ssc-sdp-mu}) and (\ref{eq:noisy-l0ssc-sdp-lambda}) in Theorem~\ref{theorem::noisy-l0ssc-subspace-detection-lambda} respectively, therefore, the conclusion holds.
\end{proof}

\newpage

\subsection{Proof of Theorem~\ref{theorem::noisy-dr-l0ssc-subspace-detection}}

We need the following lemmas before presenting the proof of Theorem~\ref{theorem::noisy-dr-l0ssc-subspace-detection}.
Lemma~\ref{lemma::random-matrix-decomposition} shows  that the low rank approximation $\bar {\bX}$ is close to $\bX$ in terms of the spectral norm~\citep{Halko2011-random-matrix-decomposition}. Lemma~\ref{lemma::perturbation-distance-to-subspace-projection} presents a perturbation bound for the distance between a data point and a subspace before and after the projection $\bP$.

\begin{lemma}\label{lemma::random-matrix-decomposition}
{\rm (Corollary $10.9$ in~\citet{Halko2011-random-matrix-decomposition})}
Let $p_0 \ge 2$ be an integer and $p' = p-p_0 \ge 4$, then with probability at least $1-6e^{-p}$, the spectral norm of $\bX - \hat \bX$ is bounded by
\bals
&\ltwonorm{\bX - \hat \bX} \le C_{p,p_0},
\eals%
where
\bals
&C_{p,p_0} \defeq \big(1+17\sqrt{1+\frac{p_0}{p'}}\big) \sigma_{p_0+1} + \frac{8\sqrt{p}}{p'+1} (\sum\limits_{j > p_0} \sigma_j^2)^{\frac{1}{2}}
\eals%
and $\sigma_1 \ge \sigma_2 \ge \ldots$ are the singular values of $\bX$.
\end{lemma}

\begin{lemma}\label{lemma::perturbation-distance-to-subspace-projection}
Let $\bbeta \in \RR^n$, $\tilde \by_i = \bP \by_i$, $\bH_{\bY_{\bbeta}}$ is an external subspace of $\by_i$, $\tilde \bY_{\bbeta} = \bP(\bY_{\bbeta})$ and $\tilde \bY_{\bbeta}$ has full column rank. Then
\bals%\label{eq:perturbation-distance-to-subspace-projection}
&|d(\by_i, \bH_{\bY_{\bbeta}}) - d(\tilde \by_i, \bH_{\tilde \bY_{\bbeta}}) |
\le C_{p,p_0} + \delta + \frac{\pth{C_{p,p_0} + 2\delta \sqrt{\tilde d_k}}(1+\delta)}{\min_{1\le r \le \tilde d_k} \sigma_{\bY, r} - C_{p,p_0} - 2\delta \sqrt{\tilde d_k}}
\eals%%
for any $1 \le i \le n$ and $\by_i \in \cS_k$.
\end{lemma}
\begin{proof}
This lemma can be proved by applying Lemma~\ref{lemma::perturbation-distance-to-hyperplane}.
\end{proof}

\vspace{0.2in}

\begin{proof}[\textup{\bf Proof of Theorem~\ref{theorem::noisy-dr-l0ssc-subspace-detection}}]
For any matrix $\bA \in \RR^{p \times q}$, we first show that multiplying $\bQ$ to the left of $\bA$ would not change its spectrum. To see this, let the singular value decomposition of $\bA$ be $\bA = \bU_{\bA} \bSigma \bV_{\bA}^{\top}$ where $\bU_{\bA}$ and $\bV_{\bA}$ have orthonormal columns with $\bU_{\bA}^{\top}\bU_{\bA} = \bV_{\bA}^{\top}\bV_{\bA} = \bI$. Then $\bQ\bA = \bU_{\bQ\bA} \bSigma \bV_{\bQ\bA}$ is the singular value decomposition of $\bQ\bA$ with $\bU_{\bQ\bA} = \bQ\bU_{\bA}$ and $\bV_{\bQ\bA} = \bV_{\bA}$. This is because the columns of $\bU_{\bQ\bA}$ are orthonormal since the columns $\bQ$ are orthonormal: $\bU_{\bQ\bA}^{\top} \bU_{\bQ\bA} = \bU_{\bA}^{\top}\bQ^{\top} \bQ\bU_{\bA} = \bI$, and $\bSigma$ is a diagonal matrix with nonnegative diagonal elements. It follows that $\sigma_{\min}(\bQ\bA) = \sigma_{\min}(\bA)$ for any $\bA \in \RR^{p \times q}$.

\vspace{0.1in}

For a point $\bx_i = \by_i + \bn_i$, after projection via $\bP$, we have the projected noise $\tilde \bn_i = \bP \bn_i$. Because
\bals
&\ltwonorm{\tilde \bn_i} = \ltwonorm{\bP \bn_i } = \ltwonorm{\bQ^{\top} \bn_i} \le \ltwonorm{\bQ} \ltwonorm{\bn_i} \le \ltwonorm{\bn_i} \le \delta,
\eals%%
the magnitude of the noise in the projected data is also bounded by $\delta$. Also,
\bals
&\ltwonorm{\tilde \bx_i} = \ltwonorm{\bQ^{\top} \bx_i } \le  \ltwonorm{\bx_i} \le 1.
\eals%%
\newpage

Let $\bbeta \in \RR^n$, $\tilde \bY_{\bbeta} = \bP \bY_{\bbeta} $ with $\|\bbeta\|_0 = r$. Since $\sigma_{\min}(\bQ \tilde \bY_{\bbeta}) = \sigma_{\min}(\tilde \bY_{\bbeta}))$, we have
\bal\label{eq:noisy-dr-l0ssc-subspace-detection-seg2}
\abth{\sigma_{\min}(\tilde \bY_{\bbeta}) - \sigma_{\min}(\bY_{\bbeta})} &= \abth{\sigma_{\min}(\bQ \tilde \bY_{\bbeta}) - \sigma_{\min}(\bY_{\bbeta})} \nonumber \\
&\le \ltwonorm{\bQ \tilde \bY_{\bbeta} - \bY_{\bbeta}} \nonumber \\
& = \ltwonorm{\bQ \bQ^{\top}\bY_{\bbeta}  - \bY_{\bbeta}}\nonumber \\
&= \ltwonorm{\bQ \bQ^{\top}\bX_{\bbeta}  - \bX_{\bbeta} + \bN_{\bbeta} - \bQ \bQ^{\top} \bN_{\bbeta} } \nonumber \\
&\le C_{p,p_0} + \norm{\bN_{\bbeta}}{F} + \norm{\bQ \bQ^{\top} \bN_{\bbeta}}{F} \nonumber \\
&\le C_{p,p_0} + 2\delta \sqrt{r}.
\eal%%

It follows from (\ref{eq:noisy-dr-l0ssc-subspace-detection-seg2}) that if
\bals
&C_{p,p_0} + 2\delta \sqrt{\tilde d_{\max}} < \min_{k \in [K]} \sigma^{(k)}_{\bY},
\eals%%
then $\tilde \bY = \bP \bY^{(k)}$ is also in general position.

Based on (\ref{eq:noisy-dr-l0ssc-subspace-detection-seg2}) we have
\bal\label{eq:noisy-dr-l0ssc-subspace-detection-seg4}
&\abth{{\bar \sigma}_{\tilde \bY,r} - {\bar \sigma}_{\bY,r}} \le C_{p,p_0} + 2\delta,
\eal%%

and it follows by (\ref{eq:noisy-dr-l0ssc-subspace-detection-seg4}) that $\delta < \min_{1 \le r < r_0} {\bar \sigma}_{\tilde \bY,r}$ because $\delta < \min_{1 \le r < r_0} \bar \sigma_{\bY,r} - C_{p,p_0} - 2\delta$.

Again, for $\bbeta \in \RR^n$ with $\|\bbeta\|_0 = r \le r_0$, we have
\bsal\label{eq:noisy-dr-l0ssc-subspace-detection-seg5}
\abth{\sigma_{\min}(\tilde \bX_{\bbeta}) - \sigma_{\min}(\bX_{\bbeta})} &= \abth{\sigma_{\min}(\bQ \tilde \bX_{\bbeta}) - \sigma_{\min}(\bX_{\bbeta})} \nonumber \\
&\le \ltwonorm{\bQ \tilde \bX_{\bbeta} - \bX_{\bbeta}} \nonumber \\
& = \ltwonorm{\bQ \bQ^{\top}\bX_{\bbeta}  - \bX_{\bbeta}}  = \ltwonorm{\hat \bX_{\bbeta} - \bX_{\bbeta}} \nonumber \\
&\le C_{p,p_0}.
\esal%%

It can be verified by (\ref{eq:noisy-dr-l0ssc-subspace-detection-seg5}) that
\bal\label{eq:noisy-dr-l0ssc-subspace-detection-seg7}
&\abth{\sigma_{\tilde \bX,r} - \sigma_{\bX,r}} \le C_{p,p_0}.
\eal%%

Combining (\ref{eq:noisy-dr-l0ssc-subspace-detection-seg7}), Lemma~\ref{lemma::perturbation-distance-to-subspace-projection}, and the known condition that
\bsals
&M_i - \pth{C_{p,p_0} + \delta + \frac{\pth{C_{p,p_0} + 2\delta \sqrt{\tilde d_k}}(1+\delta)}{\min_{1\le r \le \tilde d_k} \sigma_{\bY, r} - C_{p,p_0} - 2\delta \sqrt{\tilde d_k}}  }
> \delta + \frac{2\delta}{\sigma_{\bX,r_0} -C_{p,p_0} },
\esals%%
we have
\bals
&\tilde M_{i, \delta} \defeq \tilde M_i - \delta  > \frac{2\delta}{\tilde \sigma_{\tilde \bX,r_0}},
\eals%%
where $\by_i \in \cS_k$.

Based on (\ref{eq:noisy-dr-l0ssc-subspace-detection-seg4}) and (\ref{eq:noisy-dr-l0ssc-subspace-detection-seg7}), we have
\bals
&\tilde \mu_{r_0} < 1-\frac{2\delta}{\sigma_{\tilde \bX,r_0}},
\eals%%
because
\bals
&\frac{\delta} {\min_{1 \le r < r_0}  \bar \sigma_{\bY,r_0} - C_{p,p_0} - 3\delta}
< 1-\frac{2\delta}{\sigma_{\bX,r_0} - C_{p,p_0}}.
\eals%%

\end{proof}

\subsection{Proof of Theorem~\ref{theorem::noisy-dr-l0ssc-subspace-detection-csp}}
\label{sec::noisy-dr-l0ssc-subspace-detection-osnap-proof}
\begin{proof}[\textup{Proof of Theorem~\ref{theorem::noisy-dr-l0ssc-subspace-detection-csp}}]
First of all, $\bP$ is a $(1 \pm \eps)$ $\ell^2$-subspace embedding for the clean data matrix $\bY$. That is, when $p = \cO(\frac{r'^2}{\eps^2 \delta'})$ and $\delta' \in (0,1)$, then with probability at least $1-\delta'$, $\ltwonorm{\bP \by} \in \bth{(1 - \eps) \ltwonorm{\by}, (1 + \eps) \ltwonorm{\by}}$ for all $\by \in \RR^d$ in the column space of $\bY$.

It can be verified that $\tilde M_i \ge (1-\eps) M_i$. Moreover, let $\bbeta \in \RR^n$, $\tilde \bY_{\bbeta} = \bP \bY_{\bbeta} $ with $\|\bbeta\|_0 = r$ and ${\rm rank}(\bY_{\bbeta})=r \le r_0$. By Courant-Fischer-Weyl min-max principle for singular values and the definition of $\ell^2$-subspace embedding, we have
\bsal\label{eq:noisy-dr-l0ssc-subspace-detection-osnap-seg1}
\abth{\sigma_{t}(\tilde \bY_{\bbeta}) - \sigma_{t}(\bY_{\bbeta})} \le \eps \sigma_{t}(\bY_{\bbeta}), t \in [r].
\esal%

The following results follow from (\ref{eq:noisy-dr-l0ssc-subspace-detection-osnap-seg1}). First, the condition $\eps (1+\delta) < \sigma^{(k)}_{\bY}$ guarantees that $\tilde \bY = \bP \bY^{(k)}$ is in general position. It also follows from (\ref{eq:noisy-dr-l0ssc-subspace-detection-osnap-seg1}) that $\abth{{\bar \sigma}_{\tilde \bY,r} - {\bar \sigma}_{\bY,r}} \le \eps (1+\delta)$, and $\sigma_{\min}(\tilde \bY_{\bbeta}) \ge \sigma_{\min} (\bY_{\bbeta}) - \eps (1+\delta) > 0$ with sufficiently small $\eps$. We further have

\bsals
&\abth{\sigma_{\min} (\tilde \bX_{\bbeta}) - \sigma_{\min} (\bX_{\bbeta})}
\le  \abth{\sigma_{\min} (\tilde \bX_{\bbeta}) - \sigma_{\min} (\tilde \bY_{\bbeta})} + \abth{\sigma_{\min} (\tilde \bY_{\bbeta}) - \sigma_{\min} (\bY_{\bbeta})} + \abth{\sigma_{\min} (\bY_{\bbeta}) - \sigma_{\min} (\bX_{\bbeta})} \\
&\le \ltwonorm{\bP \bN_{\bbeta}} + \eps (1+\delta) + \ltwonorm{\bN_{\bbeta}} \\
&\le \sqrt{r_0} (1+\eps)\delta + \eps (1+\delta) + \sqrt{r_0} \delta = \delta \pth{2 \sqrt{r_0} + \eps(\sqrt{r_0}+1)} + \eps.
\esals%
where we use the fact that $\abth{\sigma_{\min} (\tilde \bY_{\bbeta}) - \sigma_{\min} (\bY_{\bbeta})} \le \eps (1+\delta)$ based on (\ref{eq:noisy-dr-l0ssc-subspace-detection-osnap-seg1}).
$\sigma_{\min} (\tilde \bX_{\bbeta}) \ge \sigma_{\min} (\bX_{\bbeta}) - \pth{\delta \pth{2 \sqrt{r_0} + \eps(\sqrt{r_0}+1)} + \eps} > 0$ for $\bbeta \in \RR^n, \|\bbeta\|_0 = r$ and ${\rm rank}(\bX_{\bbeta})=r \le r_0$ with sufficiently small $\eps$ and $\delta$. It can be verified that if (\ref{eq:noisy-dr-l0ssc-sdp-M})-(\ref{eq:noisy-dr-l0ssc-sdp-min-lambda-2}) hold, then the conditions (\ref{eq:noisy-l0ssc-sdp-M})-(\ref{eq:noisy-l0ssc-sdp-min-lambda-2}) required by Theorem~\ref{theorem::noisy-l0ssc-subspace-detection-lambda} on the projected data  also hold. Therefore, the subspace detection property holds with $\tilde \bbeta^*$ for $\tilde \bx_i$ with probability at least $1-\delta'$.
\end{proof}

\section{Bound for Suboptimal and Globally Optimal Solutions for Noisy $\ell^{0}$-SSC and Noisy-DR-$\ell^{0}$-SSC}
\label{sec::suboptimal-optimal}

While our theoretical analysis for noisy $\ell^{0}$-SSC and Noisy-DR-$\ell^{0}$-SSC is based on optimal solution to the $\ell^{0}$ regularized problem (\ref{eq:noisy-l0ssc-i}), in this subsection we prove that the bound for the suboptimal solution $\hat \bbeta$ obtained by Algorithm~\ref{alg:PGD-l0ssc} is in fact close to an optimal solution to (\ref{eq:noisy-l0ssc-i}), justifying the theoretical findings of noisy $\ell^{0}$-SSC and Noisy-DR-$\ell^{0}$-SSC.

We further present the bound for the gap between $\hat \bbeta$ and $\bbeta^*$, $\|{\hat \bbeta} - \bbeta^*\|_2$, based on Theorem 5 in~\citet{YangY19-fast-pgd-l0}. Let $g(\bbeta) = \|\bx_i - \bX {\bbeta}\|_2^2$ and $\bbeta^*$ be the globally optimal solution to (\ref{eq:noisy-l0ssc-i}), $\bS^* = {\rm supp}(\bbeta^*)$, $\hat \bbeta$ be the suboptimal solution to (\ref{eq:noisy-l0ssc-i}) obtained by PGD, $\hat \bS = {\rm supp}(\hat \bbeta)$. The following theorem presents the bound for $\|{\hat \bbeta} - \bbeta^*\|_2$.
 \begin{theorem}\label{theorem::suboptimal-optimal}
{\rm (Theorem 5 in~\citet{YangY19-fast-pgd-l0})}

Suppose $\bX_{\bS \cup \bS^*}$ has full column rank with $\kappa_0 \defeq \sigma_{\min}(\bX_{\bS \cup \bS^*}) > 0$ where $\bS$ is the support of the initialization for PGD on problem (\ref{eq:noisy-l0ssc-i}). Let $\kappa > 0$ such that $2\kappa_0^2 > \kappa$ and $b$ is chosen according to (\ref{eq:b-cond}) as below:
\bal\label{eq:b-cond}
&0< b < \min\{\min_{j \in {\hat \bS}} \abth{ \hat \bbeta_j}, \frac{\lambda}{ \max_{j \notin {\hat \bS}} \abth{\frac{\partial g}{\partial {\bbeta_j}}\vert_{\bbeta =  { \hat \bbeta}}}},
\min_{j \in {\bS^*}} \abth{\bbeta_j^*}, \frac{\lambda}{ \max_{j \notin \bS^*} \abth{\frac{\partial g}{\partial {\bbeta_j}}\vert_{\bbeta = \bbeta^*} }}  \}.
\eal%%
Let $\bF = ({\hat \bS} \setminus \bS^*) \cup (\bS^* \setminus {\hat \bS})$ be the symmetric difference between $\hat \bS$ and $\bS^*$, then
\bals%\label{eq:suboptimal-optimal}
\ltwonorm{{\hat \bbeta} - \bbeta^*} &\le \frac{1}{2\kappa_0^2-\kappa} \pth{ \sum\limits_{j \in \bF \cap \hat \bS} \pth{\max\{0,\frac{\lambda}{b} - {\kappa} \abth{{\hat \bbeta_j} - b}\}}^2  + \sum\limits_{j \in \bF \setminus \hat \bS} \pth{\max\{0,\frac{\lambda}{b} - {\kappa} b\}}^2 }^{\frac{1}{2}}.
\eals%%
\end{theorem}
\begin{remark}
It is observed that the gap $\ltwonorm{{\hat \bbeta} - \bbeta^*}$ is small when $\frac{\lambda}{b} - {\kappa} \abth{{\hat \bbeta_j} - b}$ for $j \in \bF \cap \hat \bS$ and $\frac{\lambda}{b} - {\kappa} b$ are small. Based on this observation, Theorem~\ref{theorem::suboptimal-optimal-concrete} establishes the conditions under which $\hat \bbeta$ is also an optimal solution to (\ref{eq:noisy-l0ssc-i}), i.e. ${\hat \bbeta} = \bbeta^*$.
\end{remark}

Define
\bsals
\bS^* &\defeq {\rm supp}(\bbeta^*) ,\\
H^* &\defeq \max_{j \in [n]} {\rm dist} (\bbeta, \bH_{\bX_{\bS^* \setminus \{j\}}}), \\
\mu &\defeq \max\{H^* + \ltwonorm{\bbeta_i-\bX \bbeta^*}, 2\ltwonorm{\bx_i-\bX \hat \bbeta}, 2\ltwonorm{\bx_i-\bX \bbeta^*}\}, \\
\kappa_0 &\defeq \sigma_{\min}(\bX_{\bS \cup \bS^*}) > 0,
\esals%
where $\bS = {\rm supp}({\bbeta}^{(0)})$ in the last definition. The following theorem demonstrates that $\hat \bbeta = \bbeta^*$ if $\lambda$ is two-side bounded and $\hat \bbeta_{\min} = \min_{t: \hat \bbeta_t \neq 0} \abth{\hat \bbeta_t}$ is sufficiently large.

\begin{theorem}\label{theorem::suboptimal-optimal-concrete}
{\rm (Conditions that the suboptimal solution by PGD is also globally optimal)}
If
\bal\label{eq:hat-bz-min-cond}
&\hat \bbeta_{\min} \ge \frac{\mu}{\kappa_0^2}
\eal%
and
\bal\label{eq:lambda-two-side}
&\frac{\mu^2}{2\kappa_0^2} \le \lambda \le (\hat \bbeta_{\min} - \frac{\mu}{2\kappa_0^2}) \mu,
\eal%
then ${\hat \bbeta} = \bbeta^*$.
\end{theorem}
\begin{proof}[\textup {Sketch of Proof}]
It can be verified that $\max\{0,\frac{\lambda}{b} - {\kappa} \abth{{\hat \bbeta_j} - b}\} = 0$ an $\max\{0,\frac{\lambda}{b} - {\kappa} b\} = 0$ under the conditions (\ref{eq:hat-bz-min-cond}) and (\ref{eq:lambda-two-side}), therefore, $\hat \bbeta = \bbeta^*$ by applying Theorem~\ref{theorem::suboptimal-optimal}.
\end{proof}

\end{document}